\documentclass[IEEEtran]{IEEEtran}
\IEEEoverridecommandlockouts
\usepackage{cite}
\usepackage{amsmath,amssymb,amsfonts}
\usepackage{algorithmic}
\usepackage{graphicx}
\usepackage{float}
\usepackage{textcomp}
\usepackage{xcolor}
\usepackage{svg}
\usepackage{caption}
\usepackage[numbers,sort&compress]{natbib}
\usepackage{booktabs}
\usepackage{subcaption}
\usepackage[normalem]{ulem}
\usepackage{stfloats}  % 允许双栏浮动体出现在页面底部
\usepackage{amsthm}
\makeatletter
\renewenvironment{proof}[1][\proofname]{%
  \par\pushQED{\qed}%
  \normalfont
  \topsep6\p@\@plus6\p@\relax
  \trivlist
    \item[\hskip\labelsep\bfseries #1\@addpunct{.}]\enspace
}{%
  \popQED\endtrivlist
}
\makeatother

\definecolor{ForestGreen}{RGB}{34,139,34}

\newtheorem{proposition}{Proposition}

\newtheorem{problem}{Problem}

\usepackage{pgfplots}
\usetikzlibrary{patterns,arrows,positioning,calc,fadings,shapes,decorations.markings}
\usepackage{tikz}
\usepackage{bm}
\def\BibTeX{{\rm B\kern-.05em{\sc i\kern-.025em b}\kern-.08em
    T\kern-.1667em\lower.7ex\hbox{E}\kern-.125emX}}

\begin{document}

\title{AMO-HEAD: Adaptive MARG-Only Heading Estimation for UAVs under Magnetic Disturbances

\author{Qizhi Guo, Siyuan Yang, Junning Lyu, Jianjun Sun, Defu Lin, Shaoming He\textsuperscript{*}}

\thanks{This work was supported by the National Natural Science Foundation of China under Grant No. 52302449.}
\thanks{Qizhi Guo, Siyuan Yang, Junning Lyu, Defu Lin, Shaoming He are with the School of Aerospace Engineering, Beijing Institute of Technology, Beijing, 100081, China.}
\thanks{Jianjun Sun is with the Beijing Aerospace Automatic Control Institute, Beijing, 100854, China.}
\thanks{\textsuperscript{*}Corresponding author. Email: \texttt{shaoming.he@bit.edu.cn}}
}

\maketitle

\begin{abstract}
Accurate and robust heading estimation is crucial for unmanned aerial vehicles (UAVs) when conducting indoor inspection tasks. However, the cluttered nature of indoor environments often introduces severe magnetic disturbances, which can significantly degrade heading accuracy. To address this challenge, this paper presents an \uline{A}daptive \uline{M}ARG-\uline{O}nly \uline{Head}ing (AMO-HEAD) estimation approach for UAVs operating in magnetically disturbed environments.
AMO-HEAD is a lightweight and computationally efficient Extended Kalman Filter (EKF) framework that leverages inertial and magnetic sensors to achieve reliable heading estimation.
In the proposed approach, gyroscope angular rate measurements are integrated to propagate the quaternion state, which is subsequently corrected using accelerometer and magnetometer data. The corrected quaternion is then used to compute the UAV's heading.
An adaptive process noise covariance method is introduced to model and compensate for gyroscope measurement noise, bias drift, and discretization errors arising from the Euler method integration. To mitigate the effects of external magnetic disturbances, a scaling factor is applied based on real-time magnetic deviation detection.
A theoretical observability analysis of the proposed AMO-HEAD is performed using the Lie derivative. Extensive experiments were conducted in real world indoor environments with customized UAV platforms. The results demonstrate the effectiveness of the proposed algorithm in providing precise heading estimation under magnetically disturbed conditions.
\end{abstract}


\section{Introduction}
The deployment of autonomous UAVs for inspection tasks within industrial facilities is becoming increasingly prevalent \cite{wang2024deep}. A critical requirement for these applications is precise heading \cite{wang2022bioinspired}, which is essential for maintaining a stable sensor field-of-view on inspection targets. Unlike roll and pitch, which can be accurately determined using, UAVs heading typically relies on magnetometer readings of the Earth's magnetic field. However, confined indoor environments are often characterized by significant magnetic disturbances originating from industrial equipment. These disturbances corrupt magnetometer measurements and degrade heading accuracy \cite{xu2021double}. Consequently, achieving robust heading estimation under these conditions is paramount for reliable autonomous inspection.

\begin{figure}[htp]
    \centering
    \includegraphics[width=1.0\linewidth]{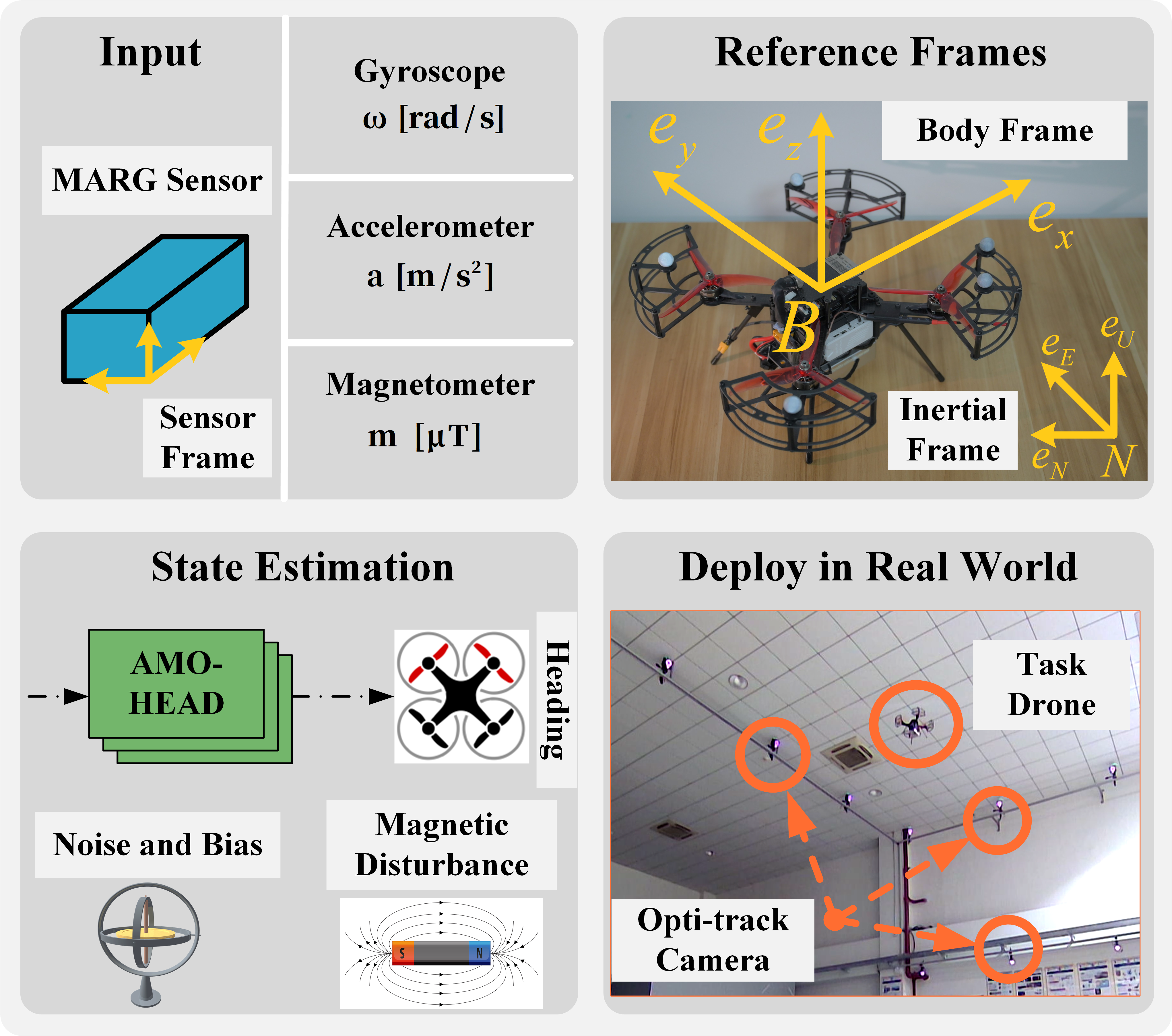}
    \caption{The illustration of the AMO-HEAD heading estimation framework. Raw gyroscope, accelerometer and magnetometer measurements from a MARG sensor are fused in real time while actively suppressing magnetic disturbances. The resulting robust heading estimates enable reliable UAV inspection under magnetic field distortions.}
    % \vspace{-0.3cm}
    \label{schematic_graph}
\end{figure}

The absence of reliable satellite navigation signals indoors \cite{he2020integrated} precludes the use of satellite-based heading determination. Alternative modalities, such as LiDAR and vision, also present challenges in industrial settings due to insufficient strong and stable features. LiDAR-based methods \cite{wang2023sw, quenzel2021real} can fail in feature-sparse environments, a problem exacerbated by the limited field of view of some sensors. Similarly, visual-inertial odometry (VIO) systems \cite{wicaksono2025go, mur2017orb, wei2022fast} demand robust visual features and significant computational resources, limiting their feasibility on small, resource-constrained UAVs. {Given these limitations, the fusion of data from MARG sensors, i.e., Magnetic, Angular Rate, and Gravity sensors, offers a compelling alternative for UAV heading estimation \cite{valenti2015keeping, kok2019fast, liu2021fast}.} MARG sensors provide lightweight and self-contained heading estimation independent of environmental features. However, estimators relying solely on an Inertial Measurement Unit (IMU) accumulate unbounded heading errors over time due to uncompensated gyroscope bias drift \cite{araguas2015quaternion}. While filter-based fusion of IMU and magnetometer data can mitigate long-term drift, conventional methods often depend on fixed noise-covariance parameters \cite{suh2012quaternion}, rendering them unable to adapt to dynamic magnetic disturbances. Therefore, the accurate modeling and compensation of these dynamic errors are critical for enhancing sensor fusion reliability.

This paper presents AMO-HEAD, as shown in Fig. \ref{schematic_graph}, a quaternion-based adaptive Extended Kalman Filter (AEKF) framework designed for robust UAV heading estimation under magnetically disturbed indoor environments. The proposed method achieves real-time performance on embedded processors by focusing on magnetic disturbance suppression without reliance on external references. The key contributions of this work are threefold: 
\begin{itemize}
\item The development of a lightweight, computationally efficient AEKF that fuses MARG sensor data to achieve accurate heading estimation.
\item The formulation of an adaptive process noise covariance model to account for gyroscope measurement noise, bias drift, and discretization errors, coupled with a scaling factor that dynamically adjusts the measurement noise covariance in response to detected magnetic disturbances.
\item An extensive experimental validation on customized UAV platforms with synchronized ground truth, demonstrating the robustness of AMO-HEAD in cluttered and magnetically disturbed indoor environments.
\end{itemize}

The rest of this paper is organized as follows. Section \ref{sec:related_works} reviews related research. Section \ref{sec:problem_formulation} presents the theoretical background. Section \ref{sec:state_estimation} details the proposed algorithm and its key components. Section \ref{sec:observability_analysis} provides the system observability analysis. Section \ref{sec:experiment_validation} reports on real world experiments and benchmark comparisons, followed by conclusions in Section \ref{sec:conclusion}.

\section{Related Works} \label{sec:related_works} 
{
Approaches to MARG sensor heading estimation fall into two categories, filter-based and learning-based methods.
% Filter-based methods can be broadly categorized into complementary filter (CF) and Kalman filter (KF) based methods.
Filter-based methods comprise complementary filter (CF) and Kalman filter (KF) variants built on explicit kinematic models, whereas learning-based methods infer heading directly from data.
We review representative CF/KF approaches and then briefly outline recent learning-based studies, positioning our adaptive EKF within this landscape.
}

\subsection{Complementary Filter-Based Heading Estimation}
CF-based method is one of the most prevalent approaches to achieve heading estimation in robotics applications. These optimization-based methods \cite{wohle2020steadeye, tong2021robust, wei2025robust} iteratively minimize an error function to refine the orientation estimate and leverage gradient descent to fine-tune parameters through iterative process. For instance, the authors in \cite{madgwick2011estimation} introduced an efficient orientation filter using a gradient descent algorithm (GDA) to find an optimal quaternion from accelerometer and magnetometer data, which then corrects the gyroscope-propagated orientation. While effective, the iterative optimization and manually tuned gain parameters can impose a significant computational burden. Later in \cite{wilson2019formulation}, the authors enhanced the algorithm in \cite{madgwick2011estimation} for controlling a robotic manipulator by decoupling magnetic field variance from roll and pitch estimation. However, this implementation relied on substantial computational resources with a large onboard computer.
{
In the underwater domain, a similar CF-based strategy requires a costly Fiber Optic Gyroscope (FOG) to maintain heading accuracy under magnetic disturbances \cite{costanzi2016attitude}.
}
% , making it impractical for real-time execution on resource-constrained UAVs. 
The work \cite{marantos2015uav} presented an orientation estimation framework for UAVs using an adaptive CF implemented directly on the $SO(3)$ manifold. The suggested adaptive mechanism adjusts MARG sensor contributions in response to external magnetic disturbances and data inconsistencies. Nevertheless, A key limitation of such $SO(3)$-based CFs is their reliance on heuristically tuned parameters and their lack of an explicit model for sensor noise statistics.

\subsection{Kalman Filter-Based Heading Estimation}
In contrast, KF-based methods are probabilistic state estimators that recursively update the orientation based on a system model and incoming measurements \cite{huang2017kalman}. Since the computation complexity of the CF grows rapidly with the number of system dimensions and sensor data, 
{KF and its nonlinear variants are often preferred for their rigorous handling of system dynamics and noise statistics \cite{xu2017optimal}.}
The authors in \cite{kada2016uav} applied an Unscented Kalman Filter (UKF) for attitude and bias estimation on a fixed-wing UAV using Euler angle parameterization. However, UKF-based algorithms are still computationally intensive. An EKF-based attitude estimation method for quadrotors with a switching strategy that selects among multiple filter modes to adapt to different flight dynamics is proposed in \cite{xu2017ekf}, achieving both high accuracy and computational efficiency for onboard implementation.
A critical limitation of conventional KF-based approaches is their reliance on a priori, static noise covariance matrices \cite{lin2023three}. 
To address this, some have explored detect-and-reject mechanisms, as demonstrated on ground robots for mitigating magnetic disturbances \cite{lee2017compensated}. This approach is ultimately limited by its reliance on rigid, empirically-tuned thresholds that fail to generalize to new environments.

In parallel with the filter-based pipelines, recent data-driven methods learn heading from MARG sensor signals. 
MARG-supervised models report low heading errors under harsh conditions when training and deployment domains are matched \cite{hoang2022yaw, li2024robust, bo2023robust}.
IMU denoising networks reduce gyro drift and yield accurate heading estimates \cite{huang2022mems, brossard2020denoising}.
Besides, heading can be inferred within data-driven odometry, which learn motion dynamics directly from inertial streams \cite{herath2020ronin, esfahani2019orinet, liu2020tlio}. 
While accurate in domain, these approaches typically require substantial labeled data and can suffer generalization issues across robot platforms, sampling rates, and especially indoor magnetic environments. 
In contrast to data-driven methods, this paper proposes AMO-HEAD. It is a training-free and physics-guided framework that explicitly fuses MARG sensor measurements. Rather than using static covariances in conventional EKFs, AMO-HEAD enhances robustness by adaptively tuning the gyroscope process noise online. Furthermore, it scales the magnetometer measurement covariance based on residuals to suppress transient outliers without hard rejection. This lightweight and adaptive design requires no training, enabling robust, real-time heading estimation on resource-constrained UAVs.

% To address this limitations, this paper proposes an AEKF-based heading estimation framework.
% It is training-free and physics-guided model that explicitly fuses the MARG sensor measurements. 
% Conventional EKFs rely on fixed noise parameters; in time-varying magnetic fields, static or mis-tuned covariances can degrade accuracy or cause divergence. AMO-HEAD adapts the gyroscope process noise online and scales the magnetometer measurement covariance from residuals to suppress transient outliers without hard rejection. This design remains real-time and requires no training, enabling robust heading on resource-constrained UAVs.
% By contrast, our training-free, physics-guided MARG-only EKF explicitly fuses the magnetometer and adaptively inflates its measurement covariance under disturbance, enabling robust heading on resource-constrained UAVs.

% The accuracy of the state estimate is highly dependent on the fidelity of these predefined noise parameters. In dynamic operational environments with unstable magnetic fields, an inaccurate or static covariance matrix can lead to suboptimal performance or filter divergence.
% To address this limitation, this paper proposes an AEKF-based heading estimation framework. Our approach fuses MARG sensor measurements and introduces an adaptive process noise covariance model to enhance filter robustness. To mitigate the influence of measurement outliers, a scaling factor mechanism is also incorporated to dynamically adjust the measurement noise covariance based on transient magnetic disturbances.

\section{Preliminaries and Backgrounds} \label{sec:problem_formulation}

Since direct measurement of the heading angle using magnetic sensors is highly susceptible to environmental interference, especially in indoor spaces filled with electronic equipment. An alternative and robust approach is to estimate the full 3-D orientation and then extract the heading component. This section first introduces the quaternion kinematics and the MARG sensor models. Finally, the problem to be solved of this paper will be presented.

\subsection{Quaternion Kinematics Model}
% Two coordinate systems are used in this paper, inertial and body coordinate frame systems, as shown in Fig. \ref{schematic_graph}. The inertial frame {$\{n\}$} is defined as the east-north-up (ENU) convention, where $x$-, $y$-, $z$-axes point east $\mathbf e_E$, north $\mathbf e_N$, and up $\mathbf e_U$ respectively. 
% The body frame $\{b\}$ is rigidly attached to the UAV airframe, while the sensor frame is defined by the MARG sensor module itself. Because the MARG unit is mounted on the flight controller with its axes aligned to the UAV's forward, left, and upward directions, the sensor frame axes and body frame axes $\mathbf e_x$, $\mathbf e_y$, $\mathbf e_z$ coincide. Consequently, any vector expressed in the sensor frame is directly valid in the body frame without additional transformation.
{
This work utilizes an inertial frame $\{n\}$ and a body frame $\{b\}$ to describe the UAV's motion, as shown in Fig. \ref{schematic_graph}. The orientation of the body frame relative to the inertial frame is represented by quaternions. Detailed descriptions of the coordinate systems are provided in Appendix \ref{appendix:verbosity}.
}
The true orientation of the UAV in the inertial frame is represented by the MARG sensor frame relative to the inertial frame, which is specified by the unit quaternion
\begin{equation}
\mathbf q = [q_w\ q_x\ q_y\ q_z]^T
\end{equation}
where $q_w$ is the scalar part and $q_x$, $q_y$, $q_z$ are the vector parts of the quaternion. 
Unlike Euler angles, quaternions are chosen for the attitude representation since they provide a singularity-free representation of 3-D rotations and hence avoid the ambiguity (such as gimbal lock). The continuous-time kinematics of the quaternion can be expressed by \cite{sabatini2006quaternion}
\begin{equation} \label{dynamic_q}
    \dot {\mathbf q} = \frac{1}{2} \mathbf q  \otimes \boldsymbol {\omega}
=\frac{1}{2} \mathbf \Omega(\boldsymbol {\omega}) \mathbf q
\end{equation}
where $\otimes$ denotes the quaternion multiplication; $\boldsymbol{\omega} = [\omega_x, \omega_y, \omega_z]^T \in \mathbb R^3$ stands for the angular rate; and $\mathbf \Omega(\boldsymbol{\omega})$ is a skew-symmetric matrix constructed from the angular velocity, used to express the quaternion differential equation in a linear form. It is defined as
\begin{equation}
    \mathbf \Omega(\boldsymbol {\omega}) =
\begin{bmatrix}
\mathbf 0 & -\boldsymbol {\omega}^T \\
\boldsymbol {\omega} & [\boldsymbol {\omega} \times]
\end{bmatrix}
\end{equation}
where $[\boldsymbol {\omega} \times]$ denotes the cross-product matrix with 
\begin{equation}
    [\boldsymbol {\omega} \times] =
\begin{bmatrix}
  0 & \omega_z & -\omega_y \\
  -\omega_z & 0 & \omega_x 
\\
 \omega_y & -\omega_x &   0
\end{bmatrix}
\end{equation}
{The UAV's heading angle $\psi$ is defined as the angle from the north direction to the UAV's forward direction, measured clockwise in the ENU frame.} The relationship between the UAV's heading angle $\psi$ and the quaternion is given by
\begin{equation} \label{q2y}
    \psi = \arctan\left[\frac{2(q_w q_z + q_x q_y)}{1 - 2(q_y^2 + q_z^2)}\right]
\end{equation}

%
%A 3-D vector $^n \mathbf v=\begin{bmatrix}v_x & v_y & v_z \end{bmatrix}^T$ in the inertial frame $\{n\}$ can be rotated to the body frame $\{b\}$ through the rotation matrix by\begin{equation}
%\begin{aligned}
%^b \mathbf v &= \mathbf {C(q)}\, ^n \mathbf v 
%\end{aligned}
%\end{equation}
%where $^b \mathbf v$ indicates the 3-D vector in the body frame.

\subsection{MARG Sensor Model}
The MARG sensor unit consists of a gyroscope, an accelerometer and a magnetometer. 
Let $\boldsymbol{\omega}_g  \in \mathbb R^3,\ \mathbf y_a \in \mathbb R^3,\ \mathbf y_m = \begin{bmatrix}  \mathrm y_{mx} & \mathrm y_{my} & \mathrm y_{mz} \end{bmatrix} \in \mathbb R^3$ be the outputs of the aforementioned sensors, respectively.
Gyroscopes are subject to inherent bias, which refers to a slowly varying offset present in the output even when there is no actual rotation. Mathematically, the measured angular velocity can be modeled as 
\begin{equation} \label{eq:gyro_noise}
\boldsymbol{\omega}_g = \boldsymbol{\omega} + \mathbf{b}_{g} + \mathbf v_{\mathbf{\omega}}
\end{equation}
where $\mathbf{b}_{g}$ represents the gyroscope bias, and $\mathbf v_{\boldsymbol{\omega}}$ is the zero-mean Gaussian measurement noise with covariance as $\mathbf{\Sigma}_{\mathbf \omega}$. 
The gyroscope bias kinematics is modeled as a white noise driven random walk
\begin{equation} \label{eq:bias_dynamics}
\mathbf {\dot b}_{g} = {\mathbf v} _b
\end{equation}
where $\mathbf v _b$ represents zero-mean Gaussian white noise with covariance $\mathbf{\Sigma}_{b}$.

The accelerometer and magnetometer provide measurements that can be modeled as observations of the gravity and local magnetic field vectors, which can be transformed into the body frame by current rotation matrix using
\begin{equation} \label{eq:project}
\begin{aligned}
\mathbf y_a &= \mathbf h_{acc}(\mathbf q)=\mathbf{C(q)} \,\mathbf{g} + \mathbf v_a \\
\mathbf y_m &= \mathbf h_{mag}(\mathbf q)=\mathbf{C(q)} \,\mathbf{m} + \mathbf v_m
\end{aligned}
\end{equation}
where the sensor noises $\mathbf v_a \in \mathbb R^3, $ and $\mathbf v_m \in \mathbb R^3$ are assumed to be uncorrelated Gaussian white noise with covariance matrices $\mathbf{R}_a$ and $\mathbf{R}_{m}$, respectively. The Earth's gravity vector $\mathbf{g}$ and magnetic field vector $\mathbf{m}$ are determined by $\mathbf{g} =\begin{bmatrix} 0 & 0 & g\end{bmatrix}^T$ and $\mathbf{m} = \begin{bmatrix}m_x^r & m_y^r & m_z^r \end{bmatrix}^T$, with $g$ being the gravitational acceleration and $m^r$ being the local reference magnetic field magnitude. 
The rotation matrix $\mathbf C(\mathbf q)$ with its full expression is detailed in Appendix \ref{appendix:verbosity}.
Based on sensor models introduced in Eq.~\eqref{eq:project}, the accelerometer and magnetometer measurement model can be represented by a compact form as
\begin{equation} \label{eq:measurement_model}
\mathbf z
=
\mathbf h(\mathbf q) + \mathbf v
= 
\begin{bmatrix}
\mathbf C(\mathbf q) & \mathbf 0_{3 \times 3} \\
\mathbf 0_{3 \times 3} & \mathbf C(\mathbf q)
\end{bmatrix} 
\begin{bmatrix}
\mathbf{g} \\ \mathbf{m}
\end{bmatrix}
+ \mathbf v
\end{equation}
where $\mathbf z$ is the measurement vector that stacks the accelerometer and magnetometer readings in the body frame as $\mathbf z =\begin{bmatrix} \mathbf y_a^T & \mathbf y_m^T \end{bmatrix}^T$, and $\mathbf v$ is the measurement noise vector, defined as $\mathbf v = \begin{bmatrix} \mathbf{v}_a^T & \mathbf{v}_m^T \end{bmatrix}^T $. The measurement noise covariance matrix $\mathbf R_t$ is given by
\begin{equation}
\mathbf{R}_t = 
\begin{bmatrix}
\mathbf{R}_a & \mathbf{0}_{3\times3} \\ 
\mathbf{0}_{3\times3} & \mathbf{R}_{m}
\end{bmatrix}
\end{equation}

\subsection{Problem Description}
It follows from Eq. \eqref{q2y} that the problem of heading estimation can be reformulated as a quaternion-based orientation estimation problem. Hence, the main objective of this paper is to estimate $\mathbf q$ using the MARG sensor measurements $\boldsymbol{\omega}_g,\ \mathbf y_a,\ \mathbf y_m$. Specifically, the quaternion propagation utilizes the angular rate measurement $\boldsymbol{\omega}_g$ from the gyroscope, while the estimation update relies on the accelerometer and magnetometer measurements $\mathbf y_a,\ \mathbf y_m$. Once we have $\mathbf q$, the heading can be readily extracted from Eq. \eqref{q2y}. Since the gyroscope measurement $\bm{\omega}_g$, as defined in Eq.~\eqref{eq:gyro_noise}, is corrupted by the unknown bias $\mathbf{b}_{g}$, we augment the bias into the estimation state as $\mathbf x =\begin{bmatrix} \mathbf q^T & \mathbf {b}_g^T\end{bmatrix}^T \in \mathbb R^7$. Then, the nonlinear dynamics model can be readily formulated using Eqs. \eqref{dynamic_q} and \eqref{eq:bias_dynamics} as
\begin{equation}\label{dyn}
\dot{\mathbf x}=\mathbf f(\mathbf x)
= \begin{bmatrix}
\frac{1}{2} \mathbf \Omega(\boldsymbol {\omega}) \mathbf q \\
{\mathbf 0} 
\end{bmatrix} 
\end{equation}

\begin{problem}
Given the nonlinear dynamics \eqref{dyn} and noise-corrupted nonlinear measurement \eqref{eq:measurement_model}, design a proper EKF to estimate $\mathbf x$.
\end{problem}

Although the unknown bias can be estimated partially rejected in the quaternion propagation by using the augmented dynamics model, the gyroscope measurement noise $\mathbf v_{\mathbf{\omega}}$ still introduces additional uncertainties in the quaternion propagation. Furthermore, magnetometer measurements in indoor environments suffer from time varying magnetic disturbances. Therefore, how to characterize sensor-related process noise and adaptively compensate for real-time magnetic disturbances is key to solve Problem 1.

\section{AMO-HEAD: Adaptive State Estimation} \label{sec:state_estimation}
In this section, we present the detailed design and implementation of the proposed AMO-HEAD for UAV heading estimation. The overview of our system is shown in Fig.~\ref{flowchart}. First, the quaternion-based filter design and state space formulation are introduced. Subsequently, an adaptive strategy for characterizing process noise covariance is described to handle gyroscope noise, bias drift, and discretization errors. Finally, we develop an adaptive measurement noise covariance mechanism to compensate for varying magnetic disturbances in real time.
\begin{figure*}[!t]      % * 表示跨两栏，位置 t=top，htbp 也可以
  \centering
  % 下面两行可选，用来让图按列宽自适应缩放
  \resizebox{\textwidth}{!}{%

\tikzset{every picture/.style={line width=0.75pt}} %set default line width to 0.75pt        

\begin{tikzpicture}[x=0.75pt,y=0.75pt,yscale=-1,xscale=1]
%uncomment if require: \path (0,271); %set diagram left start at 0, and has height of 271

%Rounded Rect [id:dp08801613246665585] 
\draw  [fill={rgb, 255:red, 247; green, 244; blue, 244 }  ,fill opacity=1 ][line width=1.5]  (5,20.02) .. controls (5,14.8) and (9.23,10.57) .. (14.45,10.57) -- (645.55,10.57) .. controls (650.77,10.57) and (655,14.8) .. (655,20.02) -- (655,240.84) .. controls (655,246.06) and (650.77,250.29) .. (645.55,250.29) -- (14.45,250.29) .. controls (9.23,250.29) and (5,246.06) .. (5,240.84) -- cycle ;
%Rounded Rect [id:dp3378802658652409] 
\draw  [fill={rgb, 255:red, 244; green, 180; blue, 180 }  ,fill opacity=1 ][dash pattern={on 3.75pt off 3pt}][line width=1.5]  (11.14,37.45) .. controls (11.14,33.1) and (14.67,29.57) .. (19.02,29.57) -- (485.27,29.57) .. controls (489.62,29.57) and (493.14,33.1) .. (493.14,37.45) -- (493.14,77.7) .. controls (493.14,82.05) and (489.62,85.57) .. (485.27,85.57) -- (19.02,85.57) .. controls (14.67,85.57) and (11.14,82.05) .. (11.14,77.7) -- cycle ;
%Rounded Rect [id:dp0025835982909252486] 
\draw  [fill={rgb, 255:red, 218; green, 242; blue, 190 }  ,fill opacity=1 ][dash pattern={on 3.75pt off 3pt}][line width=1.5]  (13.96,105.07) .. controls (13.96,100.26) and (17.86,96.36) .. (22.67,96.36) -- (614.61,96.36) .. controls (619.42,96.36) and (623.32,100.26) .. (623.32,105.07) -- (623.32,149.58) .. controls (623.32,154.39) and (619.42,158.29) .. (614.61,158.29) -- (22.67,158.29) .. controls (17.86,158.29) and (13.96,154.39) .. (13.96,149.58) -- cycle ;
%Rounded Rect [id:dp5168888132125043] 
\draw  [fill={rgb, 255:red, 198; green, 215; blue, 175 }  ,fill opacity=1 ][dash pattern={on 3.75pt off 3pt}][line width=1.5]  (523,37.45) .. controls (523,33.1) and (526.53,29.57) .. (530.87,29.57) -- (612.37,29.57) .. controls (616.71,29.57) and (620.24,33.1) .. (620.24,37.45) -- (620.24,77.7) .. controls (620.24,82.05) and (616.71,85.57) .. (612.37,85.57) -- (530.87,85.57) .. controls (526.53,85.57) and (523,82.05) .. (523,77.7) -- cycle ;
%Rounded Rect [id:dp2320739997875524] 
\draw  [color={rgb, 255:red, 0; green, 0; blue, 0 }  ,draw opacity=1 ][fill={rgb, 255:red, 165; green, 169; blue, 159 }  ,fill opacity=1 ][dash pattern={on 3.75pt off 3pt}][line width=1.5]  (11.14,178.41) .. controls (11.14,172.82) and (15.68,168.29) .. (21.27,168.29) -- (380.88,168.29) .. controls (386.47,168.29) and (391,172.82) .. (391,178.41) -- (391,230.16) .. controls (391,235.75) and (386.47,240.29) .. (380.88,240.29) -- (21.27,240.29) .. controls (15.68,240.29) and (11.14,235.75) .. (11.14,230.16) -- cycle ;
%Rounded Rect [id:dp9077608953800204] 
\draw  [fill={rgb, 255:red, 165; green, 169; blue, 159 }  ,fill opacity=1 ][dash pattern={on 3.75pt off 3pt}][line width=1.5]  (404.32,178.7) .. controls (404.32,173.13) and (408.83,168.62) .. (414.4,168.62) -- (609.92,168.62) .. controls (615.49,168.62) and (620,173.13) .. (620,178.7) -- (620,230.21) .. controls (620,235.77) and (615.49,240.29) .. (609.92,240.29) -- (414.4,240.29) .. controls (408.83,240.29) and (404.32,235.77) .. (404.32,230.21) -- cycle ;
%Shape: Ellipse [id:dp6357283663180616] 
\draw  [fill={rgb, 255:red, 248; green, 231; blue, 28 }  ,fill opacity=1 ] (545,125.07) .. controls (545,114.3) and (559.55,105.57) .. (577.5,105.57) .. controls (595.45,105.57) and (610,114.3) .. (610,125.07) .. controls (610,135.84) and (595.45,144.57) .. (577.5,144.57) .. controls (559.55,144.57) and (545,135.84) .. (545,125.07) -- cycle ;
%Rounded Rect [id:dp3986396630454354] 
\draw  [fill={rgb, 255:red, 165; green, 169; blue, 159 }  ,fill opacity=1 ][dash pattern={on 3.75pt off 3pt}][line width=1.5]  (628,33.54) .. controls (628,31.9) and (629.33,30.57) .. (630.97,30.57) -- (646.15,30.57) .. controls (647.79,30.57) and (649.12,31.9) .. (649.12,33.54) -- (649.12,236.6) .. controls (649.12,238.24) and (647.79,239.57) .. (646.15,239.57) -- (630.97,239.57) .. controls (629.33,239.57) and (628,238.24) .. (628,236.6) -- cycle ;
%Rounded Rect [id:dp27740581803587083] 
\draw  [fill={rgb, 255:red, 74; green, 144; blue, 226 }  ,fill opacity=1 ] (20,53.14) .. controls (20,51.22) and (21.56,49.66) .. (23.49,49.66) -- (145.52,49.66) .. controls (147.44,49.66) and (149,51.22) .. (149,53.14) -- (149,70.95) .. controls (149,72.87) and (147.44,74.43) .. (145.52,74.43) -- (23.49,74.43) .. controls (21.56,74.43) and (20,72.87) .. (20,70.95) -- cycle ;
%Rounded Rect [id:dp59256022808114] 
\draw  [fill={rgb, 255:red, 74; green, 144; blue, 226 }  ,fill opacity=1 ] (173,53.77) .. controls (173,51.85) and (174.56,50.29) .. (176.49,50.29) -- (310,50.29) .. controls (311.92,50.29) and (313.48,51.85) .. (313.48,53.77) -- (313.48,71.57) .. controls (313.48,73.5) and (311.92,75.05) .. (310,75.05) -- (176.49,75.05) .. controls (174.56,75.05) and (173,73.5) .. (173,71.57) -- cycle ;
%Image [id:dp33611178576683287] 
\draw (36.74,61.92) node  {\includegraphics[width=16.39pt,height=14.53pt]{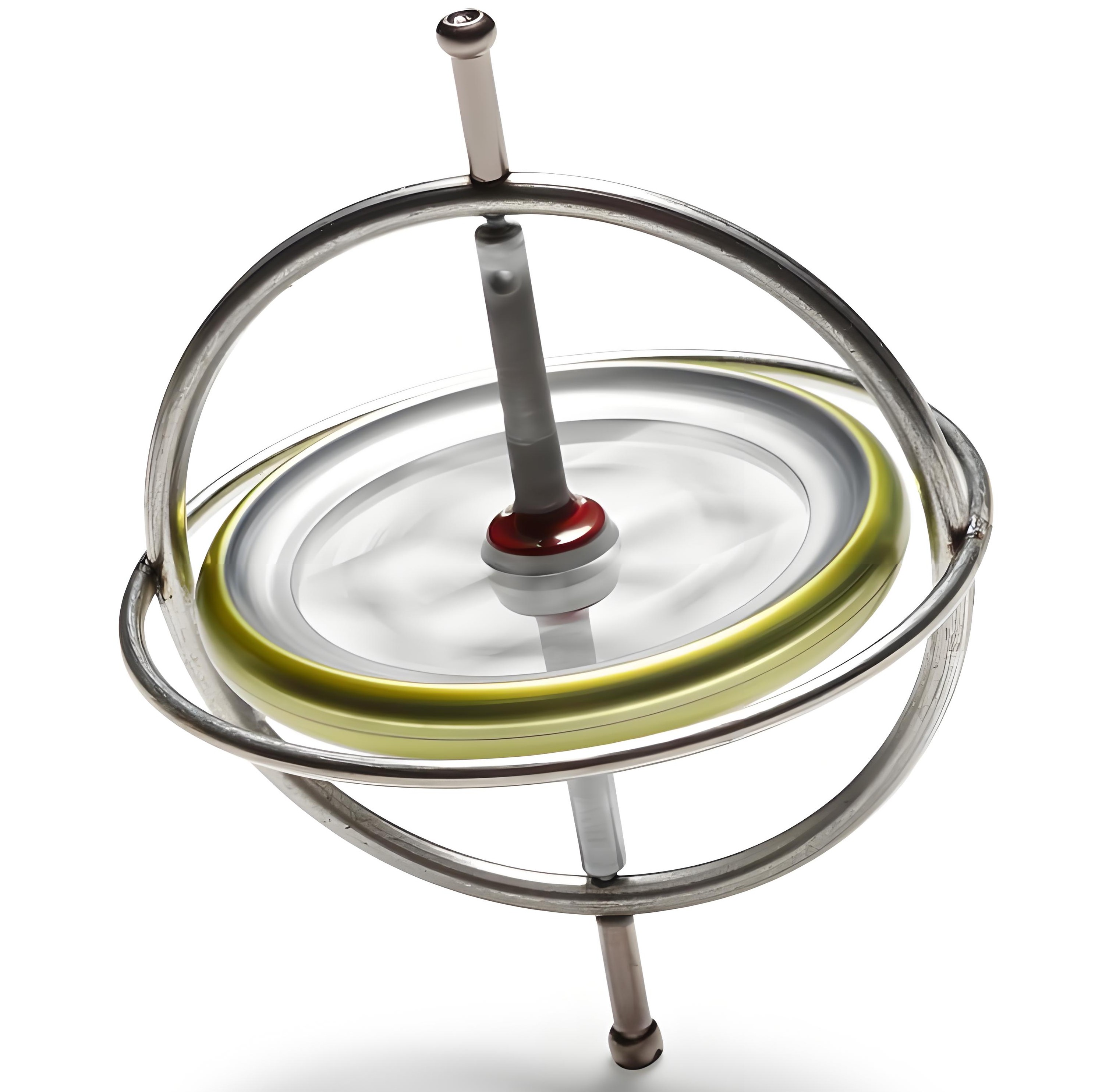}};
%Image [id:dp463704225101227] 
\draw (188.5,62.36) node  {\includegraphics[width=15.75pt,height=13.82pt]{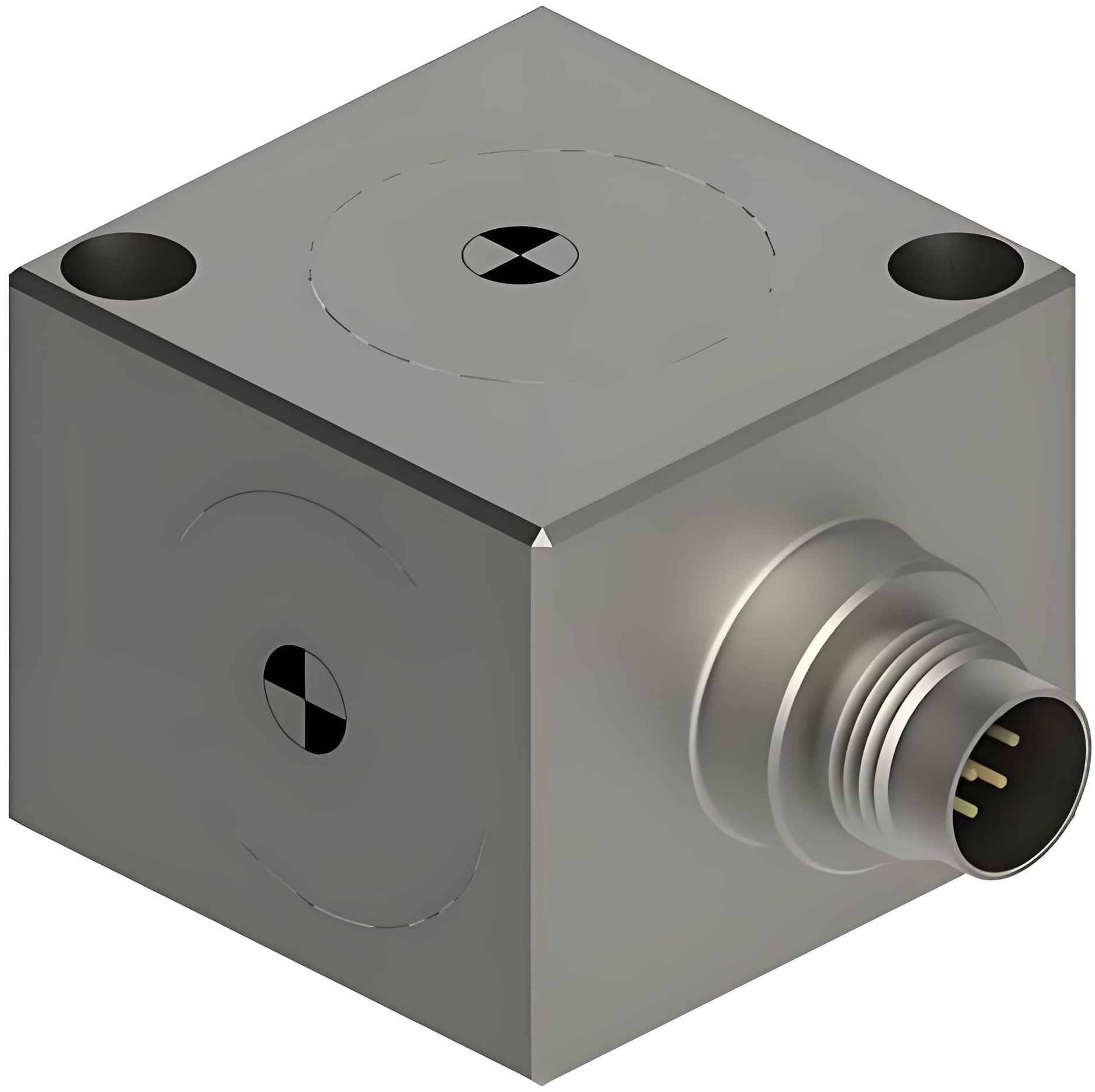}};
%Rounded Rect [id:dp00003361078892272218] 
\draw  [fill={rgb, 255:red, 74; green, 144; blue, 226 }  ,fill opacity=1 ] (337,53.77) .. controls (337,51.85) and (338.56,50.29) .. (340.49,50.29) -- (474,50.29) .. controls (475.92,50.29) and (477.48,51.85) .. (477.48,53.77) -- (477.48,71.57) .. controls (477.48,73.5) and (475.92,75.05) .. (474,75.05) -- (340.49,75.05) .. controls (338.56,75.05) and (337,73.5) .. (337,71.57) -- cycle ;
%Image [id:dp8960237583955157] 
\draw (353.63,62.57) node  {\includegraphics[width=16.71pt,height=15.43pt]{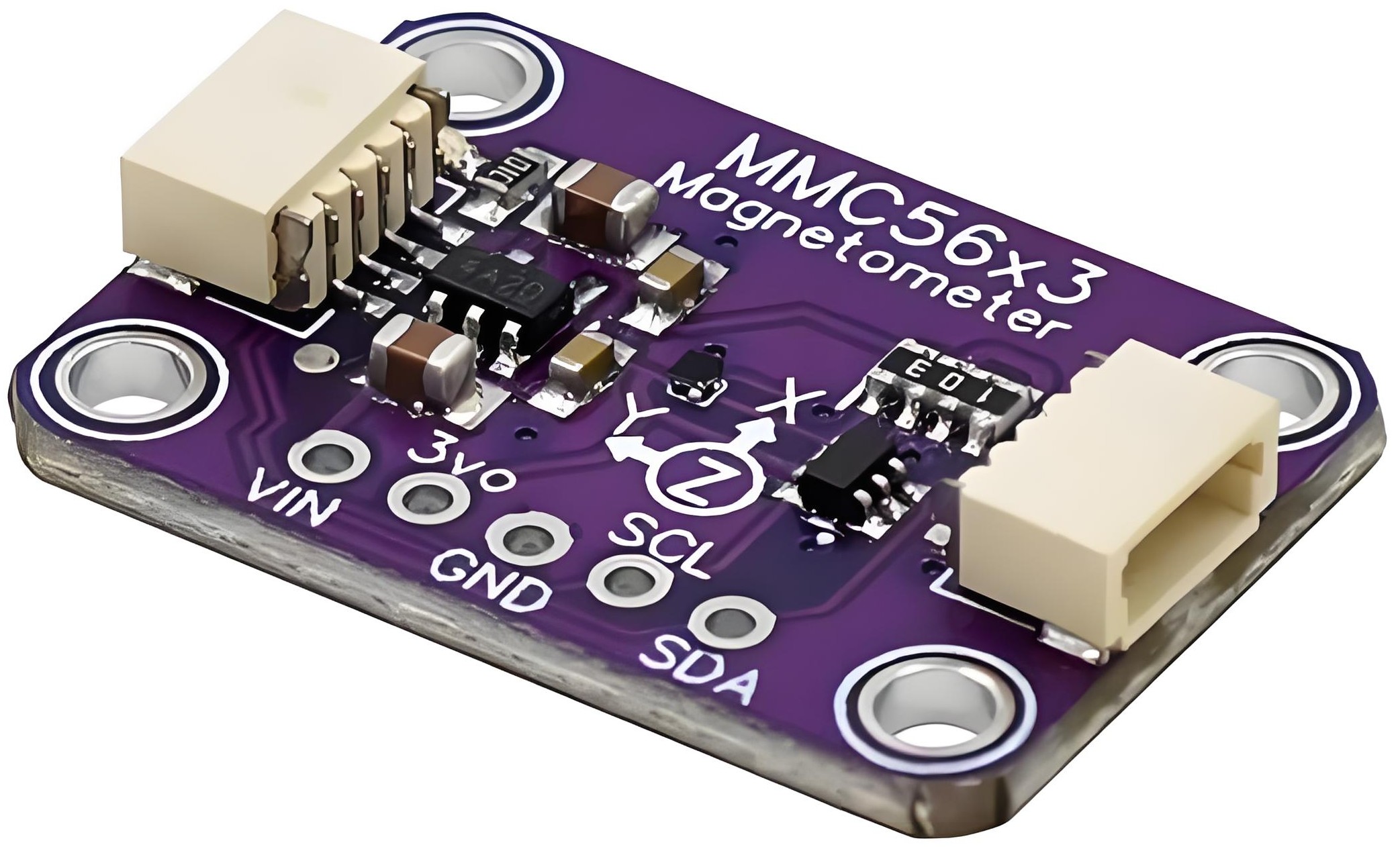}};
%Rounded Rect [id:dp9986429325070736] 
\draw  [dash pattern={on 3.75pt off 1.5pt on 1.5pt off 1.5pt}][line width=1.0]  (166.14,50.35) .. controls (166.14,47.71) and (168.28,45.57) .. (170.92,45.57) -- (479.36,45.57) .. controls (482,45.57) and (484.14,47.71) .. (484.14,50.35) -- (484.14,74.79) .. controls (484.14,77.43) and (482,79.57) .. (479.36,79.57) -- (170.92,79.57) .. controls (168.28,79.57) and (166.14,77.43) .. (166.14,74.79) -- cycle ;
%Rounded Rect [id:dp3892816420012284] 
\draw  [fill={rgb, 255:red, 245; green, 166; blue, 35 }  ,fill opacity=1 ] (26,114.65) .. controls (26,112.4) and (27.83,110.57) .. (30.08,110.57) -- (150.92,110.57) .. controls (153.17,110.57) and (155,112.4) .. (155,114.65) -- (155,135.49) .. controls (155,137.75) and (153.17,139.57) .. (150.92,139.57) -- (30.08,139.57) .. controls (27.83,139.57) and (26,137.75) .. (26,135.49) -- cycle ;
%Rounded Rect [id:dp2381601893412385] 
\draw  [fill={rgb, 255:red, 245; green, 166; blue, 35 }  ,fill opacity=1 ] (331,114.63) .. controls (331,112.39) and (332.82,110.57) .. (335.06,110.57) -- (424.94,110.57) .. controls (427.18,110.57) and (429,112.39) .. (429,114.63) -- (429,135.37) .. controls (429,137.61) and (427.18,139.43) .. (424.94,139.43) -- (335.06,139.43) .. controls (332.82,139.43) and (331,137.61) .. (331,135.37) -- cycle ;
%Rounded Rect [id:dp8760664386929552] 
\draw  [fill={rgb, 255:red, 245; green, 166; blue, 35 }  ,fill opacity=1 ] (177,114.77) .. controls (177,112.45) and (178.88,110.57) .. (181.2,110.57) -- (301.8,110.57) .. controls (304.12,110.57) and (306,112.45) .. (306,114.77) -- (306,136.23) .. controls (306,138.55) and (304.12,140.43) .. (301.8,140.43) -- (181.2,140.43) .. controls (178.88,140.43) and (177,138.55) .. (177,136.23) -- cycle ;
%Rounded Rect [id:dp4505188650889257] 
\draw  [fill={rgb, 255:red, 245; green, 166; blue, 35 }  ,fill opacity=1 ] (453,114.63) .. controls (453,112.39) and (454.82,110.57) .. (457.06,110.57) -- (519.94,110.57) .. controls (522.18,110.57) and (524,112.39) .. (524,114.63) -- (524,135.37) .. controls (524,137.61) and (522.18,139.43) .. (519.94,139.43) -- (457.06,139.43) .. controls (454.82,139.43) and (453,137.61) .. (453,135.37) -- cycle ;
%Rounded Rect [id:dp9296746775121] 
\draw  [fill={rgb, 255:red, 198; green, 215; blue, 175 }  ,fill opacity=1 ] (20,197.87) .. controls (20,194.95) and (22.38,192.57) .. (25.31,192.57) -- (128.7,192.57) .. controls (131.63,192.57) and (134,194.95) .. (134,197.87) -- (134,224.98) .. controls (134,227.91) and (131.63,230.29) .. (128.7,230.29) -- (25.31,230.29) .. controls (22.38,230.29) and (20,227.91) .. (20,224.98) -- cycle ;
%Rounded Rect [id:dp18992211662645875] 
\draw  [fill={rgb, 255:red, 198; green, 215; blue, 175 }  ,fill opacity=1 ] (143,197.63) .. controls (143,194.68) and (145.39,192.29) .. (148.35,192.29) -- (251.66,192.29) .. controls (254.61,192.29) and (257,194.68) .. (257,197.63) -- (257,224.94) .. controls (257,227.89) and (254.61,230.29) .. (251.66,230.29) -- (148.35,230.29) .. controls (145.39,230.29) and (143,227.89) .. (143,224.94) -- cycle ;
%Rounded Rect [id:dp042353410023330706] 
\draw  [fill={rgb, 255:red, 198; green, 215; blue, 175 }  ,fill opacity=1 ] (269,197.77) .. controls (269,194.74) and (271.46,192.29) .. (274.49,192.29) -- (379.66,192.29) .. controls (382.69,192.29) and (385.14,194.74) .. (385.14,197.77) -- (385.14,225.8) .. controls (385.14,228.83) and (382.69,231.29) .. (379.66,231.29) -- (274.49,231.29) .. controls (271.46,231.29) and (269,228.83) .. (269,225.8) -- cycle ;
%Rounded Rect [id:dp6764848987056734] 
\draw  [fill={rgb, 255:red, 198; green, 215; blue, 175 }  ,fill opacity=1 ] (410,190.1) .. controls (410,188.47) and (411.32,187.15) .. (412.95,187.15) -- (611.19,187.15) .. controls (612.82,187.15) and (614.14,188.47) .. (614.14,190.1) -- (614.14,205.19) .. controls (614.14,206.82) and (612.82,208.15) .. (611.19,208.15) -- (412.95,208.15) .. controls (411.32,208.15) and (410,206.82) .. (410,205.19) -- cycle ;
%Rounded Rect [id:dp8946034766011227] 
\draw  [fill={rgb, 255:red, 198; green, 215; blue, 175 }  ,fill opacity=1 ] (411,216.52) .. controls (411,214.89) and (412.32,213.57) .. (413.95,213.57) -- (611.19,213.57) .. controls (612.82,213.57) and (614.14,214.89) .. (614.14,216.52) -- (614.14,231.62) .. controls (614.14,233.25) and (612.82,234.57) .. (611.19,234.57) -- (413.95,234.57) .. controls (412.32,234.57) and (411,233.25) .. (411,231.62) -- cycle ;
%Straight Lines [id:da9638893890943406] 
\draw [color={rgb, 255:red, 245; green, 166; blue, 35 }  ,draw opacity=1 ][line width=1.5]    (155.14,125.29) -- (173.14,125.29) ;
\draw [shift={(177.14,125.29)}, rotate = 180] [fill={rgb, 255:red, 245; green, 166; blue, 35 }  ,fill opacity=1 ][line width=0.08]  [draw opacity=0] (6.97,-3.35) -- (0,0) -- (6.97,3.35) -- cycle    ;
%Straight Lines [id:da13376859281357978] 
\draw [color={rgb, 255:red, 245; green, 166; blue, 35 }  ,draw opacity=1 ][line width=1.5]    (306.14,125.29) -- (327.14,125.29) ;
\draw [shift={(331.14,125.29)}, rotate = 180] [fill={rgb, 255:red, 245; green, 166; blue, 35 }  ,fill opacity=1 ][line width=0.08]  [draw opacity=0] (6.97,-3.35) -- (0,0) -- (6.97,3.35) -- cycle    ;
%Straight Lines [id:da622259645130211] 
\draw [color={rgb, 255:red, 245; green, 166; blue, 35 }  ,draw opacity=1 ][line width=1.5]    (429.14,125.29) -- (449.14,125.29) ;
\draw [shift={(453.14,125.29)}, rotate = 180] [fill={rgb, 255:red, 245; green, 166; blue, 35 }  ,fill opacity=1 ][line width=0.08]  [draw opacity=0] (6.97,-3.35) -- (0,0) -- (6.97,3.35) -- cycle    ;
%Straight Lines [id:da170639751845319] 
\draw [color={rgb, 255:red, 245; green, 166; blue, 35 }  ,draw opacity=1 ][line width=1.5]    (524.14,125.29) -- (541,125.11) ;
\draw [shift={(545,125.07)}, rotate = 179.41] [fill={rgb, 255:red, 245; green, 166; blue, 35 }  ,fill opacity=1 ][line width=0.08]  [draw opacity=0] (6.97,-3.35) -- (0,0) -- (6.97,3.35) -- cycle    ;
%Straight Lines [id:da9389154752230271] 
\draw [color={rgb, 255:red, 165; green, 169; blue, 159 }  ,draw opacity=1 ][line width=1.5]    (240.14,167.29) -- (240.14,145.29) ;
\draw [shift={(240.14,141.29)}, rotate = 90] [fill={rgb, 255:red, 165; green, 169; blue, 159 }  ,fill opacity=1 ][line width=0.08]  [draw opacity=0] (6.97,-3.35) -- (0,0) -- (6.97,3.35) -- cycle    ;
%Straight Lines [id:da5036311570978422] 
\draw [color={rgb, 255:red, 165; green, 169; blue, 159 }  ,draw opacity=1 ][line width=1.5]    (418.14,167.29) -- (418.14,143.29) ;
\draw [shift={(418.14,139.29)}, rotate = 90] [fill={rgb, 255:red, 165; green, 169; blue, 159 }  ,fill opacity=1 ][line width=0.08]  [draw opacity=0] (6.97,-3.35) -- (0,0) -- (6.97,3.35) -- cycle    ;
%Straight Lines [id:da10400666723167673] 
\draw [color={rgb, 255:red, 74; green, 144; blue, 226 }  ,draw opacity=1 ][line width=1.5]    (484.14,61.29) -- (519.14,61.29) ;
\draw [shift={(523.14,61.29)}, rotate = 180] [fill={rgb, 255:red, 74; green, 144; blue, 226 }  ,fill opacity=1 ][line width=0.08]  [draw opacity=0] (6.97,-3.35) -- (0,0) -- (6.97,3.35) -- cycle    ;
%Straight Lines [id:da8546965928193634] 
\draw [color={rgb, 255:red, 74; green, 144; blue, 226 }  ,draw opacity=1 ][line width=1.5]    (503.64,60.29) -- (503.18,106.29) ;
\draw [shift={(503.14,110.29)}, rotate = 270.57] [fill={rgb, 255:red, 74; green, 144; blue, 226 }  ,fill opacity=1 ][line width=0.08]  [draw opacity=0] (6.97,-3.35) -- (0,0) -- (6.97,3.35) -- cycle    ;
%Straight Lines [id:da6590961664121169] 
\draw [color={rgb, 255:red, 74; green, 144; blue, 226 }  ,draw opacity=1 ][line width=1.5]    (438.14,74.29) -- (438.14,163.29) ;
\draw [shift={(438.14,167.29)}, rotate = 270] [fill={rgb, 255:red, 74; green, 144; blue, 226 }  ,fill opacity=1 ][line width=0.08]  [draw opacity=0] (6.97,-3.35) -- (0,0) -- (6.97,3.35) -- cycle    ;

% Text Node
\draw (646.84,96.13) node [anchor=north west][inner sep=0.75pt]  [rotate=-90] [align=left] {\textbf{Experiments}};
% Text Node
\draw (18.81,13.57) node [anchor=north west][inner sep=0.75pt]  [font=\footnotesize] [align=left] {\textit{\textbf{{\small System Overview}}}};
% Text Node
\draw (21.02,32.57) node [anchor=north west][inner sep=0.75pt]  [font=\footnotesize] [align=left] {\textbf{\textit{{\small MARG Sensor on UAV}}}};
% Text Node
\draw (56.49,53.57) node [anchor=north west][inner sep=0.75pt]  [font=\footnotesize] [align=left] {\textbf{{\large Gyroscope}}};
% Text Node
\draw (205,55.14) node [anchor=north west][inner sep=0.75pt]  [font=\footnotesize] [align=left] {\textbf{{\large Accelerometer}}};
% Text Node
\draw (372,55.14) node [anchor=north west][inner sep=0.75pt]  [font=\footnotesize] [align=left] {\textbf{{\large Magnetometer}}};
% Text Node
\draw (178,119.14) node [anchor=north west][inner sep=0.75pt]  [font=\footnotesize] [align=left] {\begin{minipage}[lt]{87.05pt}\setlength\topsep{0pt}
\begin{center}
\textbf{ Forward Propagation}
\end{center}

\end{minipage}};
% Text Node
\draw (343.78,119.14) node [anchor=north west][inner sep=0.75pt]  [font=\footnotesize] [align=left] {\begin{minipage}[lt]{52.6pt}\setlength\topsep{0pt}
\begin{center}
\textbf{Kalman Gain}
\end{center}

\end{minipage}};
% Text Node
\draw (469.33,119.14) node [anchor=north west][inner sep=0.75pt]  [font=\footnotesize] [align=left] {\begin{minipage}[lt]{30.38pt}\setlength\topsep{0pt}
\begin{center}
\textbf{Update}
\end{center}

\end{minipage}};
% Text Node
\draw (46,113.14) node [anchor=north west][inner sep=0.75pt]  [font=\footnotesize] [align=left] {\begin{minipage}[lt]{65.13pt}\setlength\topsep{0pt}
\begin{center}
\textbf{UAV Orientation}\\\textbf{Gyro Bias}
\end{center}

\end{minipage}};
% Text Node
\draw (22,199.02) node [anchor=north west][inner sep=0.75pt]  [font=\footnotesize] [align=left] {\begin{minipage}[lt]{76.64pt}\setlength\topsep{0pt}
\begin{center}
\textbf{Gyro Measurement}\\\textbf{Noise Propagation}
\end{center}

\end{minipage}};
% Text Node
\draw (156,199.02) node [anchor=north west][inner sep=0.75pt]  [font=\footnotesize] [align=left] {\begin{minipage}[lt]{60.31pt}\setlength\topsep{0pt}
\begin{center}
\textbf{Gyro Bias Drift}\\\textbf{Compensation}
\end{center}

\end{minipage}};
% Text Node
\draw (270,199.02) node [anchor=north west][inner sep=0.75pt]  [font=\footnotesize] [align=left] {\begin{minipage}[lt]{78.45pt}\setlength\topsep{0pt}
\begin{center}
\textbf{Discretization Error}\\\textbf{Propagation}
\end{center}

\end{minipage}};
% Text Node
\draw (19,172) node [anchor=north west][inner sep=0.75pt]  [font=\footnotesize] [align=left] {\textit{\textbf{{\small Process Noise Modeling}}}};
% Text Node
\draw (18,142.49) node [anchor=north west][inner sep=0.75pt]  [font=\footnotesize] [align=left] {\textit{\textbf{{\small State Estimation (50Hz)}}}};
% Text Node
\draw (412,193.1) node [anchor=north west][inner sep=0.75pt]  [font=\footnotesize] [align=left] {\begin{minipage}[lt]{137.56pt}\setlength\topsep{0pt}
\begin{center}
\textbf{Magnetic Disturbance Assessment}
\end{center}

\end{minipage}};
% Text Node
\draw (460,218.09) node [anchor=north west][inner sep=0.75pt]  [font=\footnotesize] [align=left] {\begin{minipage}[lt]{76.64pt}\setlength\topsep{0pt}
\begin{center}
\textbf{Dynamic R-Scaling}
\end{center}

\end{minipage}};
% Text Node
\draw (410,171) node [anchor=north west][inner sep=0.75pt]  [font=\footnotesize] [align=left] {\textit{\textbf{{\small Adaptive Measurement Cov}}}};
% Text Node
\draw (552.33,113.14) node [anchor=north west][inner sep=0.75pt]  [font=\footnotesize] [align=left] {\begin{minipage}[lt]{34.91pt}\setlength\topsep{0pt}
\begin{center}
\textbf{UAV}\\\textbf{Heading}
\end{center}

\end{minipage}};
% Text Node
\draw (533.33,45.14) node [anchor=north west][inner sep=0.75pt]  [font=\footnotesize] [align=left] {\begin{minipage}[lt]{54.42pt}\setlength\topsep{0pt}
\begin{center}
\textbf{Observability}\\\textbf{Analysis}
\end{center}

\end{minipage}};
\end{tikzpicture}
  }
  \caption{System overview of the proposed heading estimation framework.}
  \label{flowchart}
\end{figure*}
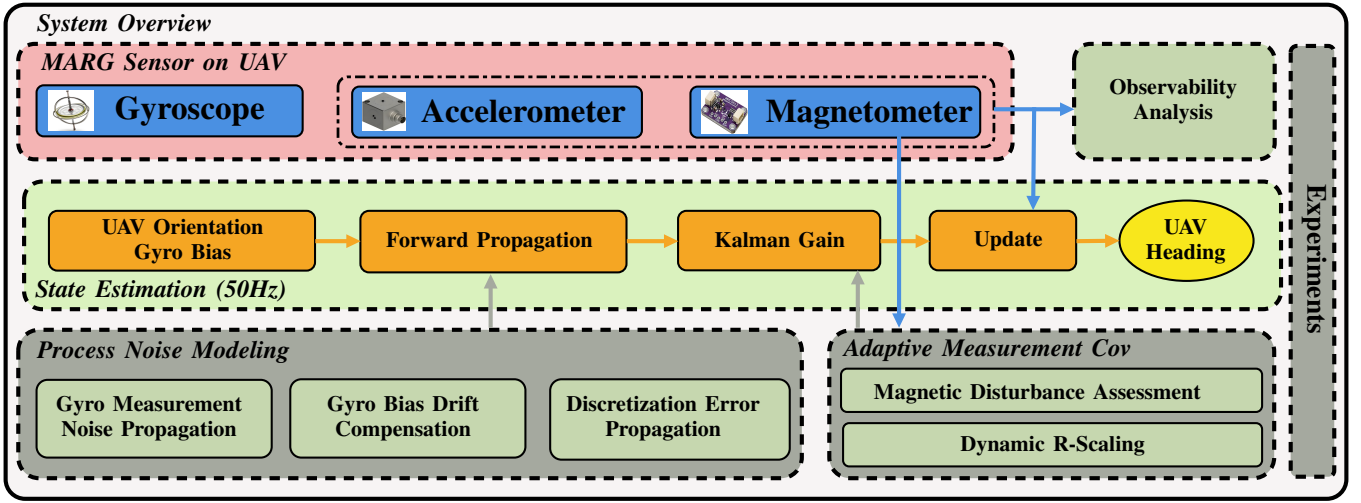

\subsection{Filter Design}\label{AA}
The magnetometer readings are fused with the accelerometer and gyroscope readings to estimate the UAV's heading by the proposed AEKF-based estimator. The AEKF operates through the standard prediction-update cycle. 

\subsubsection{Quaternion Propagation}To compute the quaternion iteration in a discrete time manner, the continuous time quaternion kinematics \eqref{dynamic_q} can then be discretized as
\begin{equation}
\mathbf q _t = \mathbf \exp \left[\frac{\Delta t}{2} \mathbf \Omega(\boldsymbol{\omega}_t)\right] \mathbf q_{t-1}
\end{equation}
where $\mathbf \exp \left[\cdot\right]$ denotes the matrix exponential and the variable with subscript $t$ stands for the corresponding value at time instant $t$. The notation $\Delta t$ represents the sample period.

Under the assumption that the angular velocity remains constant over the sampling interval $\left[t, t+ \Delta t \right]$, this exact solution incorporates all higher-order terms in $\Delta t$ via the Taylor expansion~\cite{kanwal1989taylor}
\begin{equation}
    \mathbf \exp \left[\frac{\Delta t}{2} \mathbf \Omega(\boldsymbol{\omega}_t)\right] = 
    \mathbf I + \frac{\Delta t}{2} \mathbf \Omega(\boldsymbol {\omega}_t) + \frac{\Delta t^2} {8}\mathbf \Omega^2(\boldsymbol{\omega}_t) + \mathcal O(\Delta t^3)
\end{equation}
where $\mathcal O(\Delta t^3)$ stands for higher order terms.
In order to leverage the EKF concept, we follow the Euler method \cite{cartwright1992dynamics} to linearize the quaternion kinematics by removing the second and higher order terms, and the discrete time quaternion kinematics \eqref{dynamic_q} can then be linearized as
\begin{equation} \label{discrete_q1}
\mathbf {q}_{t} =
\mathbf {q}_{t-1} + \frac{\Delta t}{2} \mathbf \Omega(\boldsymbol{\omega}_t) \mathbf {q}_{t-1}
\end{equation}
However, directly using Eq. \eqref{discrete_q1} for quaternion propagation in filter design faces two main difficulties. On one hand, the angular rate available is the measured signal $\boldsymbol{\omega}_{g,t}$ instead of the true one $\boldsymbol{\omega}_{t}$, and hence is corrected by the measurement noise, which directly affects the accuracy of model \eqref{discrete_q1}. On the other hand, the Euler approximation neglects higher-order terms and hence also imposes negative impact on the accuracy of model \eqref{discrete_q1}. To accommodate these two issues, we propose to modify dynamics model \eqref{discrete_q1} as
\begin{equation} \label{discrete_q}
\mathbf {q}_{t} =
\mathbf {q}_{t-1} + \frac{\Delta t}{2} \mathbf \Omega(\boldsymbol{\omega}_{g,t}) \mathbf {q}_{t-1} + \mathbf{w}_{q,t}, \quad \mathbf w_{q,t} \in \mathcal N(0, \mathbf Q_{q,t}) 
\end{equation}
where  the process noise $\mathbf w_{q,t} = \mathbf n_{\boldsymbol{\omega}, t} + \boldsymbol{\varepsilon}_t$ with $\mathbf n_{\boldsymbol{\omega}, t}\sim\mathcal{N}(0,\mathbf{Q}_{\boldsymbol{\omega},t})$ being the term introduced by the gyro measurement noise and $\boldsymbol{\varepsilon}_t\sim\mathcal{N}(0,\mathbf{Q}_{\mathrm{int},t})$ being introduced by the Euler approximation error. Assume that the gyroscope measurement noise and the Euler method discretization error are uncorrelated, the process noise covariance of the quaternion dynamics can be determined by
\begin{equation}
    \mathbf Q_{q,t} = \mathbf Q_{\boldsymbol{\omega},t} + \mathbf Q_{\mathrm{int},t}
\end{equation}
Based on the continuous time bias dynamics model \eqref{eq:bias_dynamics}, the discrete time bias propagation can be formulated as 
\begin{equation} \label{eq:discrete_b}
    \mathbf b_{g,t} = \mathbf b_{g,t-1} + \mathbf w_{b,t}, \quad \mathbf w_{b,t} \in \mathcal N(0, \mathbf Q_{b,t}) 
\end{equation}
where $\mathbf w_{b,t}$ represents the process noise associated with the gyroscope bias and $\mathbf Q_{b,t}$ is the corresponding covariance matrix. 

Since the attitude-propagation noise and bias drift uncertainties are independent, we define the augmented process noise covariance matrix into a unified matrix as
\begin{equation}
    \mathbf{Q}_t =
    \begin{bmatrix}
    \mathbf{Q}_{q,t} & \mathbf{0}_{4 \times 3} \\
    \mathbf{0}_{3 \times 4} & \mathbf{Q}_{b,t}
    \end{bmatrix}
\end{equation}
The prediction step in the proposed AEKF propagates the state estimate by
\begin{equation} \label{process_model}
    \mathbf{x}^-_t = \mathbf F_t \mathbf {\hat x}_{t-1} 
\end{equation}
where $\mathbf{x}^-_t$ indicates the predicted value at the current time step, and $\mathbf {\hat x}_{t-1}$ denotes the optimal estimate at the previous time step.  The state transition matrix $\mathbf F_t$ is given by
\begin{equation}
\mathbf{F}_t =
\begin{bmatrix}
\mathbf{A}_q(\boldsymbol{\tilde{\omega}}_t, \Delta t) & \mathbf{0}_{4 \times 3} \\
\mathbf{0}_{3 \times 4} & \mathbf{I}_{3}
\end{bmatrix}
\end{equation}
where $\boldsymbol{\tilde{\omega}}_t = \boldsymbol{\omega}_{g,t} - \hat{\mathbf{b}}_{g,t}$ is the bias corrected angular velocity, and matrix $\mathbf{A}_q(\boldsymbol{\tilde{\omega}}_t, \Delta t)$ is determined by
\begin{equation}
% \begin{aligned}
    \mathbf{A}_q(\boldsymbol{\tilde{\omega}}_t, \Delta t) =\mathbf{I}_4 + \frac{\Delta t}{2} \mathbf{\Omega}(\boldsymbol{\tilde{\omega}}_t)\\
%     & =\begin{bmatrix}
% 1 & -\frac{\Delta t}{2} \tilde{\omega}_x & -\frac{\Delta t}{2} \tilde{\omega}_y & -\frac{\Delta t}{2} \tilde{\omega}_z \\
% \frac{\Delta t}{2} \tilde{\omega}_x & 1 & \frac{\Delta t}{2} \tilde{\omega}_z & -\frac{\Delta t}{2} \tilde{\omega}_y \\
% \frac{\Delta t}{2} \tilde{\omega}_y & -\frac{\Delta t}{2} \tilde{\omega}_z & 1 & \frac{\Delta t}{2} \tilde{\omega}_x \\
% \frac{\Delta t}{2} \tilde{\omega}_z & \frac{\Delta t}{2} \tilde{\omega}_y & -\frac{\Delta t}{2} \tilde{\omega}_x & 1
% \end{bmatrix}
% \end{aligned}
\end{equation}
{
with its full expression provided in Appendix \ref{appendix:verbosity}.
}
% where $\mathbf {\tilde \omega}_t = \mathbf \omega_m - \mathbf b_{g}$ denotes the bias corrected angular rates. 
The covariance matrix of the process model is predicted by 
\begin{equation}
    \mathbf P_{t}^{-} = \mathbf F_{t} \mathbf{\hat{P}}_{t-1} \mathbf F_{t}^{T} + \mathbf Q_{t}
\end{equation}
where $\mathbf P_{t}^{-}$ denotes prior estimate covariance at the current time step and $\mathbf{\hat{P}}_{t-1}$ denotes posterior estimate covariance at the previous time step.

\subsubsection{Quaternion Update}During the update procedure, the filter incorporates the accelerometer and magnetometer measurements to correct the predicted quaternion and covariance matrix are updated by
\begin{equation}
\begin{aligned}
\mathbf{\hat{x}}_t &= \mathbf{x}^-_t + \mathbf{K}_t \left[ \mathbf{z}_t - \mathbf{h}(\mathbf{x}^-_t) \right]  \\ 
\mathbf{\hat P}_t &= \mathbf{P}^-_t - \mathbf{K}_t \mathbf{H}_t \mathbf{P}^-_t
\end{aligned}
\end{equation}
where $\mathbf K_t$ denotes Kalman gain and it is calculated with the linearized model
\begin{equation}
    \mathbf{K}_t = \mathbf{P}_t^- \mathbf{H}_t^T \left( \mathbf{H}_t \mathbf{P}_t^- \mathbf{H}_t^T + \mathbf{R}_t \right)^{-1}
\end{equation}
where the Jacobian matrix of the measurement model is given by
\begin{equation}
\mathbf{H}_t =\frac{\partial \mathbf h(\mathbf x_t)}{\partial \mathbf x}\bigg|_{\mathbf x_t = \mathbf x_t^-}
\end{equation}
{
and the detailed definition is presented in Appendix \ref{appendix:verbosity}.
}

By using the quaternion estimates through inertial-magnetic sensor fusion, i.e., the first four elements of the estimated state $\mathbf{\hat{x}}_t$, the heading estimation can be readily extracted from Eq. \eqref{q2y}. However, how to modeling the process noise $\mathbf w_{q,t}$ to compensate for the quaternion dynamics modeling error and adjust the magnetometer measurement covariance to account for the magnetic disturbance will be the key to improve the estimation performance. These improvements will be detailed in the next two subsections.

\subsection{Adaptive Process Noise Characterization} \label{process_noise}
In this section, the process noise model is detailed in a physics-informed manner. We model three primary sources of uncertainty: gyroscope measurement induced process noise, Euler discretization induced process noise, and bias drift. Each component is analyzed individually, and corresponding noise covariance representations are derived to construct an adaptive process noise model.

\subsubsection{Gyroscope Measurement Induced Process Noise} 
Let the noise in the gyroscope measurements be $ \mathbf v_{\mathbf \omega,t} = 
\begin{bmatrix}
v_{x,t}^\omega & v_{y,t}^\omega & v_{z,t}^ \omega
\end{bmatrix}^T $.
As the angular velocity input for the prediction step is derived from gyroscope measurements subject to stochastic errors, these uncertainties are propagated into the process noise. Therefore, incorporating $\mathbf v_{{\mathbf \omega},t}$ into the quaternion kinematics \eqref{dynamic_q} yields
\begin{equation} \label{eq:gyro_uncertainties_noise}
\begin{bmatrix}
\dot q_w \\
\dot q_x \\
\dot q_y \\
\dot q_z
\end{bmatrix}
=
\frac{1}{2}
\begin{bmatrix}
0 & -v_{x,t}^\omega & -v_{y,t}^\omega & -v_{z,t}^\omega \\
v_{x,t}^\omega & 0 & v_{z,t}^\omega & -v_{y,t}^\omega \\
v_{y,t}^\omega & -v_{z,t}^\omega & 0 & v_{x,t}^\omega \\
v_{z,t}^\omega & v_{y,t}^\omega & -v_{x,t}^\omega & 0
\end{bmatrix}
\begin{bmatrix}
q_w \\
q_x \\
q_y \\
q_z
\end{bmatrix}
\end{equation}
which presents the continuous quaternion noise caused by the gyroscope noises.
Consequently, we derive a noise model by transforming the gyroscope measurement error \eqref{eq:gyro_uncertainties_noise} into the quaternion domain as
\begin{equation}
    \mathbf n_{\boldsymbol{\omega},t} = \frac{\Delta t}{2} \mathbf \Phi(\hat{\mathbf q}_{t-1}) \mathbf v_{\omega,t}
\end{equation}
where $\mathbf \Phi(\hat{\mathbf q}_{t-1})$ is the mapping matrix that projects the 3-D angular rate uncertainties into the 4-D quaternion space, defined as
\begin{equation}
\mathbf \Phi({\mathbf q}_{t})
=
\begin{bmatrix}
-q_x & -q_y & -q_z \\
q_w & -q_z & q_y \\
q_z & q_w & -q_x \\
-q_y & q_x & q_w
\end{bmatrix}
\end{equation}
The process noise covariance matrix is computed using the gyroscope measurement noise covariance
\begin{equation}
\mathbf Q_{{\boldsymbol{\omega}},t} = \mathbb E \begin{bmatrix} \mathbf n_{\boldsymbol{\omega},t} \, \mathbf n_{\boldsymbol{\omega},t}^T \end{bmatrix}
= 
\frac{\Delta t^2}{4} \mathbf \Phi(\hat{\mathbf q}_{t-1})
\mathbf \Sigma_{\mathbf \omega} 
\mathbf \Phi^T(\hat{\mathbf q}_{t-1}) 
\end{equation}
where $\mathbb E\left[ \cdot \right]$ denotes the mathematical expectation operation.

\subsubsection{Euler Discretization Induced Process Noise}
The discrete-time approximation in Eq.~\eqref{discrete_q} truncates this series by only leveraging the first-order term, thereby introducing integration errors. It is clear that the discretization error can be reduced by compensating for the higher-order terms. To balance the accuracy and computational tractability, we propose to incorporate the second-order term to mitigate the approximation error. The resulting discretization error can be expressed as the difference between the first-order and the second-order approximations, i.e.,
\begin{equation}
    \boldsymbol{\varepsilon}_t  = \frac{\Delta t^2}{8} \,\mathbf \Omega^2(\boldsymbol {\omega}_{g,t}) \hat{\mathbf q}_{t-1} 
\end{equation}

Assuming the angular velocity is isotropically distributed, the covariance matrix of the truncation error can be approximated by  
\begin{equation}
    \mathbf Q_{{int,t}} 
    = \mathbb{E}\begin{bmatrix}
    \boldsymbol{\varepsilon}_t\, \boldsymbol{\varepsilon}_t^T \end{bmatrix} 
    \approx \frac{\Delta t^4}{64} \| \boldsymbol {\omega}_{g,t}\|^4 \,\hat{\mathbf q}_{t-1}\,\hat{\mathbf q}_{t-1}^T
\end{equation}
\subsubsection{Gyroscope Bias Error}
In practical scenarios, gyroscope measurements are affected by slowly drifting biases. {If left unmodeled, these biases accumulate over time and cause long-term drift in heading estimates \cite{xu2017distributed}.} To address this, we explicitly model the gyroscope bias kinematics as a random walk process, as described in Eq.~\eqref{eq:bias_dynamics}. 
To derive a discrete time formulation suitable for implementation, we integrate both sides of the continuous time equation over a sampling interval $\Delta t$ as
\begin{equation}
\mathbf{b}_{g,t} - \mathbf{b}_{g,t-1} = 
\int_{t-1}^{t} \dot{\mathbf{b}}_g(\tau)\, d\tau 
= \int_{t-1}^{t} \mathbf{v}_b(\tau) \, d\tau 
\end{equation}
where $\tau$ is a dummy variable of integration over the interval $\left[t-1, t \right]$.
% This integral represents the accumulated effect of continuous-time white noise over the interval $\Delta t$. 
According to Eq.~\eqref{eq:discrete_b}, the discrete time noise term $\mathbf{w}_{b,t}$ can be expressed as
\begin{equation} 
\mathbf{w}_{b,t} = 
\int_{t-1}^{t} \mathbf{v}_b(\tau) \, d\tau
\end{equation}

Under the white noise assumption, $\mathbf{w}_{b,t}$ remains zero-mean Gaussian, and its covariance is given by
\begin{align}
    \mathbf{Q}_{b,t} = \mathbf E [ \mathbf{w}_{b,t} \, \mathbf{w}_{b,t}^T ] 
    &= 
    \int_{t-1}^{t} \int_{t-1}^{t} \mathbf E [ \mathbf{v}_b(\tau) \, \mathbf{v}_b^T(\sigma) ] \, d\tau \, d\sigma \notag \\
    &= 
    \int_{t-1}^{t} \int_{t-1}^{t} \mathbf \Sigma_b \, \delta(\tau-\sigma) \, d\tau \, d\sigma \notag \\
    &= 
    \mathbf \Sigma_b \, \Delta t 
\end{align}
where $\sigma$ is a dummy variable and $\delta(\cdot)$ is the Dirac delta function.

\subsection{Adaptive Magnetometer Measurement Noise Covariance}
In indoor environments with {cluttered electronic equipment and metallic objects}, magnetometer measurements are often degraded by strong magnetic disturbances. 
{
It is important to note that our approach builds upon a one-time, offline calibration performed prior to deployment. This initial calibration compensates for sensor-intrinsic hard and soft iron biases and establishes the local reference magnetic field vector. The adaptive mechanism detailed in this section is specifically designed to handle external, time-varying disturbances encountered during operation.
}
The core idea of this stratefy is to compute the deviation between the measured magnetic field vector in the body frame and the predicted magnetic field based on the current state estimate. This deviation, denoted as $\mathbf \delta_m$, is calculated by
\begin{equation}
    \mathbf \delta_m = \|\mathbf y_{m,t} - \mathbf{h}_{mag}(\mathbf{q}_t^-)\|
\end{equation}
where $\mathbf{h}_{mag}(\mathbf{q}_t^-)$ is the predicted magnetic field vector derived from the current quaternion estimate. 
We employ a Chi-square distribution \cite{song2024innovation} based on the magnetometer measurement residual $\mathbf \delta_m$ to establish a criterion for detecting disturbances. Under nominal conditions, the squared Mahalanobis distance of the residual is defined as
\begin{equation}
d^2 = \mathbf \delta_m^T \mathbf{R}_{m}^{-1} \mathbf \delta_m
\end{equation}
which follows a Chi-square distribution with three degrees of freedom, denoted as 
\begin{equation}
d^2 \sim \chi^2_3,\quad P(d^2 \leq \chi^2_{3,p_i}) = p_i, \quad i = 1,2
\end{equation}
where $p_i$ represents predefined confidence levels associated with disturbance severities. Accordingly, thresholds $\tau_1$ and $\tau_2$ are determined by
\begin{equation}
\tau_i = \sigma_m \sqrt{\chi^2_{3,p_i}}, \quad i = 1,2
\end{equation}
where $\sigma_m = \sqrt{\frac{\text{trace}(\mathbf{R}_{m})}{3}}$ is the standard deviation of magnetometer measurement noise under nominal conditions. 
The adaptive mechanism identifies disturbances by comparing the magnetometer innovation residual $\delta_m$ against predefined thresholds, $\tau_1$ and $\tau_2$. These thresholds are derived from the chi-squared distribution, corresponding to confidence levels of $p_1 = 0.95$ for severe disturbances and $p_2 = 0.35$ for moderate ones, respectively. Based on these levels, a piecewise function scales the magnetometer noise covariance $\mathbf{R}_m$ as follows
\begin{equation}
    \mathbf{R}_{m} =
    \begin{cases} 
    \lambda_1 \mathbf{R}_{m}, & \text{if } \delta_m > \tau_1 \quad \text{(severe disturbance)} \\ 
    \lambda_2 \mathbf{R}_{m}, & \text{if } \tau_2 < \delta_m \leq \tau_1 \quad \text{(moderate disturbance)} \\ 
    \mathbf{R}_{m}, & \text{otherwise} \quad \text{(nominal conditions)}
    \end{cases}
\end{equation}

{
The disturbance scaling coefficients, $\lambda_1$ and $\lambda_2$, are empirically tuned using a validation dataset collected within the same indoor environment where the subsequent experiments are conducted. Specifically, $\lambda_1$ is a large factor chosen to aggressively inflate the measurement covariance during severe disturbances, effectively causing the filter to reject the corrupted measurement. In contrast, $\lambda_2$ is a more moderate factor used to down-weight the magnetometer's influence during transient disturbances.
}
{
The confidence levels $p_1$ and $p_2$ can also be tuned to adapt the filter to different platforms and operational environments. For UAVs operating in magnetically challenging conditions, a more conservative tuning is recommended to prioritize robustness. This is achieved by lowering both $p_1$ and $p_2$ to make the filter more sensitive to anomalies. Conversely, for platforms in relatively clean environments, accuracy can be maximized by raising $p_2$ to trust the magnetometer more frequently, while keeping $p_1$ high to ensure only genuinely extreme outliers are rejected. 
}
This adaptive strategy ensures the robustness and accuracy of the heading estimation by dynamically adjusting the magnetometer's credibility to compensate for varying magnetic conditions in complex indoor environments.

\section{Nonlinear Observability Analysis} \label{sec:observability_analysis}
In order to assess whether the heading angle can be estimated by the filter, we perform a nonlinear observability analysis. A nonlinear system is locally observable if the rank of the observability matrix $\mathbf \Xi$ equals the dimension of the state, i.e., $\operatorname{rank}(\mathbf \Xi)=\text{dim}(\mathbf x)$. Following the methods presented in \cite{hermann1977nonlinear}, an observability matrix derived from Lie derivatives is established. For the general definition, the observability matrix is constructed by stacking the gradients of successive Lie derivatives as
\begin{equation} 
\mathbf \Xi
= \begin{bmatrix}
\nabla_{\mathbf x} {\mathcal L}^0 \mathbf h(\mathbf{q}) \\
\nabla_{\mathbf x} {\mathcal L}^1_f  \mathbf h(\mathbf{q}) \\
\mathbf \vdots & \\
\nabla_{\mathbf x} {\mathcal L}^{k}_{\mathbf f}  \mathbf h(\mathbf{q})
\end{bmatrix}, \ k\le n-1
\end{equation}
where
\begin{equation}
\begin{aligned}
&{\mathcal L}^0 \mathbf h(\mathbf{q}) = \mathbf h(\mathbf{q})\\
&{\mathcal L}^{k}_{\mathbf f}  \mathbf h(\mathbf{q})=\frac{\partial {\mathcal L}^{k-1}_{\mathbf f}  \mathbf h(\mathbf{q})}{\partial \mathbf x} \mathbf f(\mathbf x)\\
&\nabla_{\mathbf x} {\mathcal L}^{k}_{\mathbf f}  \mathbf h(\mathbf{q})=\frac{\partial {\mathcal L}^{k}_{\mathbf f}  \mathbf h(\mathbf{q})}{\partial \mathbf x}
\end{aligned}
\end{equation}
and $\nabla_{\mathbf x} (\cdot)$ denotes the gradient of $(\cdot)$ with respect to $\mathbf x$ and $n$ is the dimension of the state $\mathbf x$.
In this formulation, the superscript $n-1$ denotes that Lie derivatives are computed up to the $n-1$ order to construct $n$ linearly independent gradient terms for an $n$-dimensional state vector.
In our specific case, considering the zeroth and first order Lie derivatives is sufficient to show that the observability matrix attains full rank, which is equal to the dimension of the state $\mathbf x$, i.e., the observability matrix can be established as
\begin{equation} \label{ob_matrix}
\mathbf \Xi
= \begin{bmatrix}
\nabla_{\mathbf x} {\mathcal L}^0 \mathbf h(\mathbf{q}) \\[6pt]
\nabla_{\mathbf x} {\mathcal L}^1_{\mathbf f}  \mathbf h(\mathbf{q})
\end{bmatrix}\in \mathbb R^{12\times 7}
\end{equation}

Then, the observability of the proposed filter is presented in the following proposition, which indicates the system state, including the quaternion and bias, is locally observable, theoretically guaranteeing the convergence of the filter developed.

\begin{proposition}\label{prop:fullrank}
The rank of the observability matrix $\mathbf \Xi$ associated with the proposed AMO-HEAD equals the dimension of the system state, i.e., $\operatorname{rank}(\mathbf \Xi)=7$.
\end{proposition}
\begin{proof}
A rigorous proof of the full rank property is presented in the Appendix \ref{proof_ob_matrix}. 
\end{proof}

%\begin{remark}
%Despite the theoretical possibility of rank deficiency in the observability matrix at specific singular attitudes (e.g., when the heading angle is $0^\circ, 90^\circ, \text{or } 180^\circ$), these points do not compromise the filter's practical performance or lead to instability.
%\end{remark}
%{
%This practical robustness stems from two primary factors. First, a UAV is an inherently dynamic system subject to constant, minor corrective movements. The UAV passes these singular attitudes so transiently that the filter is not exposed to any single point long enough to cause divergence. Second, our observability verification method is itself robust. As demonstrated in the analysis, global observability is confirmed by verifying the rank of different submatrices. A rank drop in one specific submatrix does not imply a total loss of observability, as another submatrix can still retain full rank. This conclusion is corroborated in the subsequent experimental section, where no filter divergence or noticeable performance degradation is observed when the UAV's trajectory passed through or near these theoretical singular attitudes.
%}

\section{Experiments Validation}\label{sec:experiment_validation}
In this section, we present a set of experiments in indoor environments subject to magnetic disturbances, as shown in Fig. \ref{exp_env}. These experiments demonstrate the effectiveness and robustness of the proposed AEKF-based heading estimation algorithm. We first describe the equipment setup and the specific UAV hardware platform. Next, we discuss the indoor magnetic environment and demonstrate how significant field disturbances affect raw magnetometer readings. Then we conduct comparative experiments using four progressively refined designs of our algorithm:
(1) EKF-4D: using only a 4-dimensional quaternion as state variables with random-walk process noise; (2) EKF-7D: expanding the state to 7 dimensions by incorporating gyroscope biases, still with random-walk noise; (3) EKF-Accurate: implementing refined modeling of process noise in addition to the 7-dimensional state; and (4) AMO-HEAD: further adding an adaptive measurement noise covariance module to the refined model in case (3). 
In addition, we compare the proposed method with several representative open-source heading estimation algorithms to benchmark its performance under magnetically disturbed conditions.
\begin{figure}[htp]
  \centering
  % 子图 (a)
  \begin{subfigure}[a]{\linewidth}
    \centering
    \includegraphics[width=0.8\linewidth]{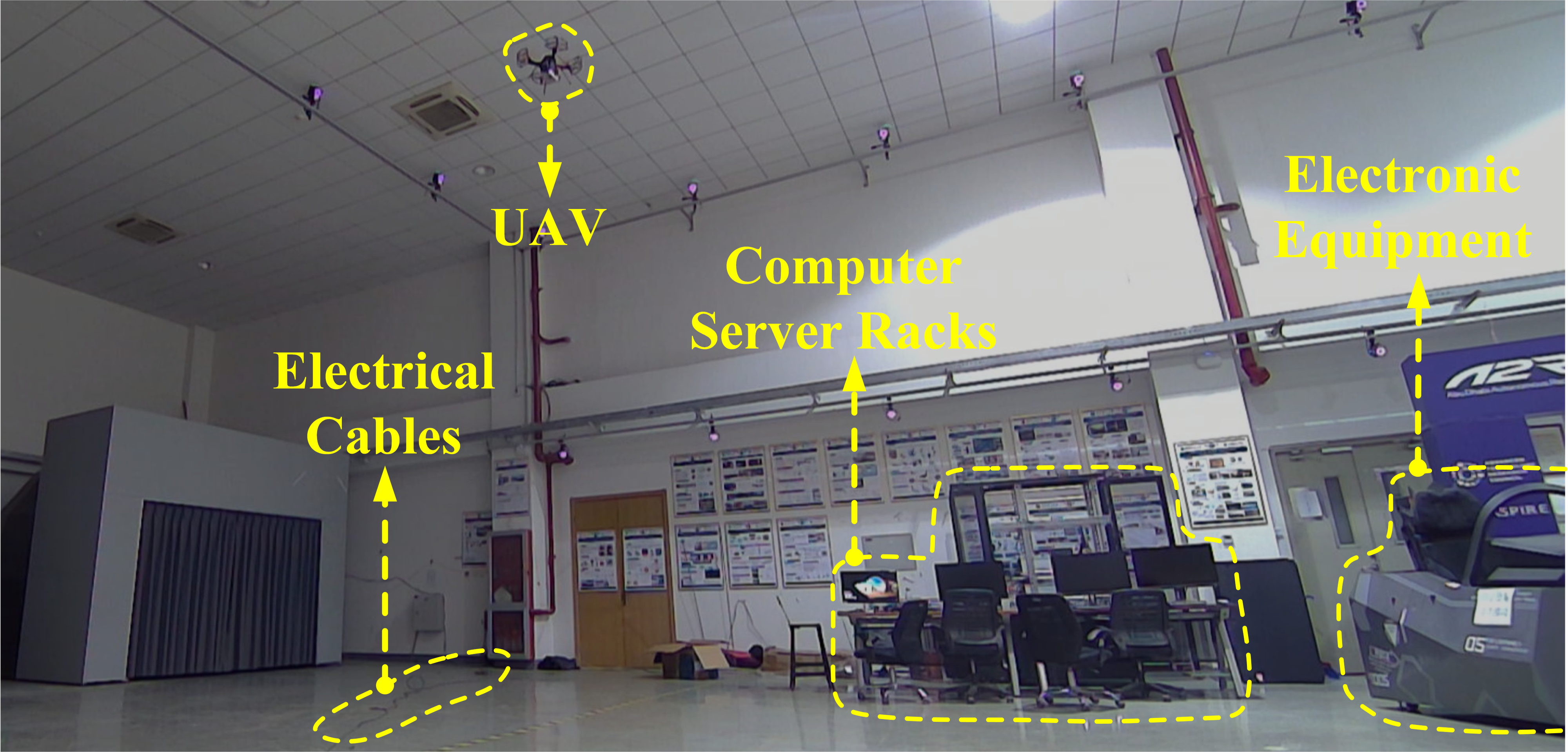}
    \caption{The indoor inspection area is filled with server racks and electronic equipment, resulting in magnetic disturbances}         % 留空只显示 (a)
    \label{exp_env}
  \end{subfigure}
  % \hfill  
  % 子图 (b)
  \begin{subfigure}[b]{\linewidth}
    \centering
    \includegraphics[width=0.8\linewidth]{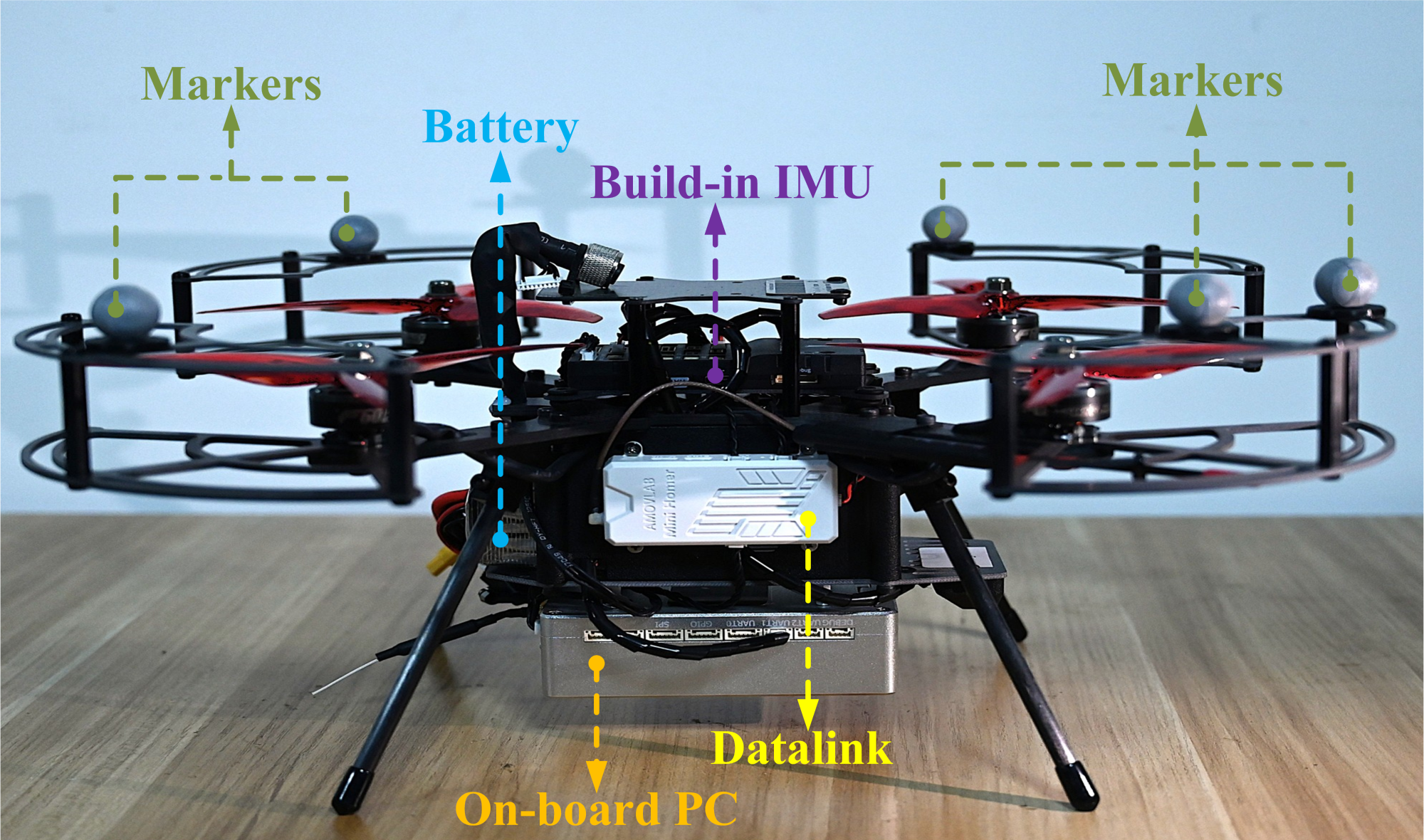}
    \caption{Our quadrotor platform with hardware details for data acquisition. The quadrotor is tested within a motion capture system}         % 留空只显示 (b)
    \label{equip_no_site}
  \end{subfigure}
  % 整图标题
  \caption{Experimental settings.}
  \label{exp_set}
\end{figure}

\subsection{Equipment Setup}
Our experiments were conducted on a customized quadrotor UAV shown in Fig.  \ref{equip_no_site}. It is equipped with a Jetson Orin NX on-board computer featuring 2.0 $GHz$ Arm Cortex-A7 CPU and 16 $GB$ RAM. The computer runs Ubuntu 20.04 with the Robot Operating System (ROS) middleware for sensor data handling and real-time control. The UAV's flight controller integrates an IMU containing three-axis accelerometers and gyroscopes, as well as a dedicated three-axis magnetometer for heading measurements.
%\begin{figure}[htp]
%    \centering
%    \includegraphics[width=1.0\linewidth]{figure/exp_uav_detail.png}
%    \caption{Our quadrotor platform with hardware details for data acquisition. The quadrotor is tested within a motion capture system.}
%    \label{equip_no_site}
%\end{figure}

To obtain ground truth pose information for validation, we employed a motion capture system composed of 44 cameras arranged in two layers: an upper layer of 24 cameras approximately $6\,\mathrm{m}$ above the ground, and a lower layer of 20 cameras situated around $3\,\mathrm{m}$ high, as shown in Fig. \ref{exp_env}. 
Covering a measurement volume of $17\,\mathrm{m}\times 17\,\mathrm{m}\times 5\,\mathrm{m}$, the system delivers centimeter-level positioning precision throughout this space and maintains sub-degree angular accuracy, making it well-suited for reliably capturing the UAV's position and orientation in real time.
In the following experiments, the motion capture system provides the reference heading as ground truth. 

\subsection{Magnetic Field Conditions and Disturbance Effects}
Our tests were performed in a lab containing electrical equipment and steel structures.
{
To quantitatively characterize this magnetically cluttered environment, we compute the magnetic field magnitude $\mathbf{B}$ over the experimental duration with 
\begin{equation}
|\mathbf{B}| = \sqrt{{\mathrm y^2_{mx}} + {\mathrm y^2_{my}} + {\mathrm y^2_{mz}}}
\end{equation}
which is shown in Fig.~\ref{fig:mag_field_strength_initial_3min_exp}. 
The summary statistics show a mean magnetic field magnitude $\mu_{|\mathbf{B}|}$ of $45.83~\mu\mathrm{T}$, with a standard deviation $\sigma_{|\mathbf{B}|}$ of $11.13~\mu\mathrm{T}$.
The coefficient of variation 
\begin{equation}
C_v=\frac{\sigma_{|\mathbf{B}|}}{\mu_{|\mathbf{B}|}}
\end{equation}
gives a value of 24.28$\%$, which markedly exceeds the 10$\%$ threshold used to assess significant magnetic deviation \cite{laidig2023vqf}. Thus, it indicates the pronounced magnetic disturbance. 
In addition, the peak magnitude reaches $127.34~\mu T$, approximately $2.78\times$ the mean value, further corroborating the presence of strong local disturbance sources. 
These local magnetic field distortions noticeably affect the raw three-axis magnetometer readings.
}
\begin{figure}[h!]
  \centering
    \includegraphics[width=0.9\linewidth]{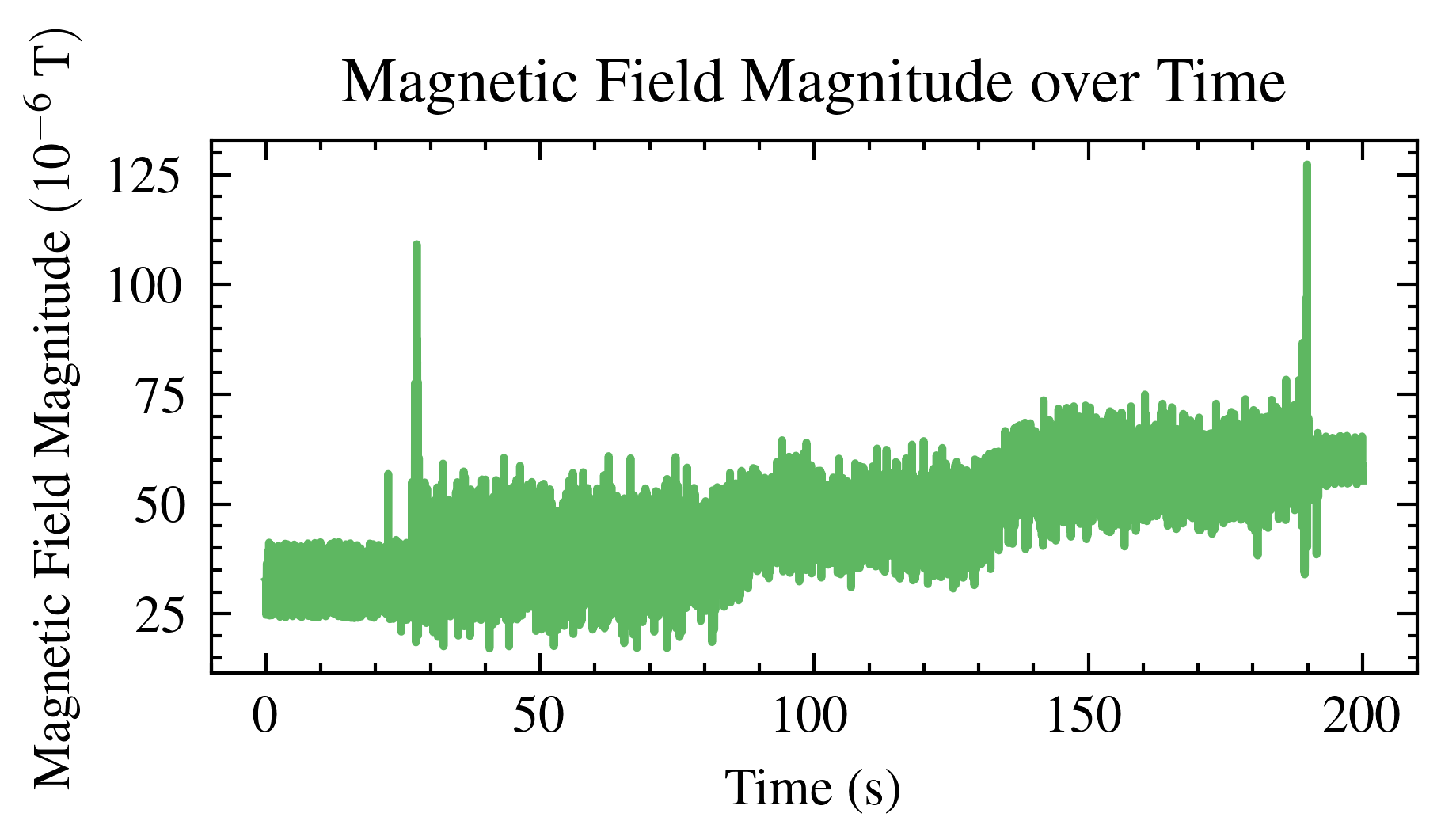}
  \caption{Magnetic field magnitude recorded in the experimental environment.}
  \label{fig:mag_field_strength_initial_3min_exp}
\end{figure}
As the UAV should rotate to scan its surroundings in order to fully cover the inspection area, we consider the scenario that involves a UAV hovering approximately $6 \,\mathrm{m}$ high, changing its heading angle without translating. Specifically, the UAV hovers steadily at a fixed altitude in indoor environment while performing controlled heading rotations. The heading angle transitions from 2.8° to 95.8°, and then to 186.5° by a human pilot, with the UAV holding each orientation for a defined duration. 
Specifically, the UAV undergoes three distinct phases: in the first phase it maintains a heading of 2.8° for 70 seconds, in the second phase 95.8° for 50 seconds, and in the third phase 186.5° for another 50 seconds.
% Specifically, the UAV maintains a heading of 2.8° for 70 seconds, 95.8° for 50 seconds, and 186.5° for another 50 seconds. 

The magnetometer readings in this scenario, shown in Fig.~\ref{fig:sub_b}, demonstrate significant variability.
At $T = 20$s, sudden magnetic anomalies are observed on both the Y and Z axes. The Y-axis abruptly spikes to 60$\mu T$, and the Z-axis jumps sharply to -100$\mu T$. Later at $T = 180$s, the Z-axis value increases to -120$\mu T$, exceeding the typical Earth's magnetic field magnitude \cite{zhang2016magnetic}.
These pronounced magnetic field disturbances captured by the magnetometer readings indicate the nature of the laboratory's magnetic environment. Such interferences confirm the necessity of adopting a robust sensor-fusion strategy instead of solely relying on the raw outputs for heading calculations.

\begin{figure}[htp]
  \centering
    \includegraphics[width=\linewidth]{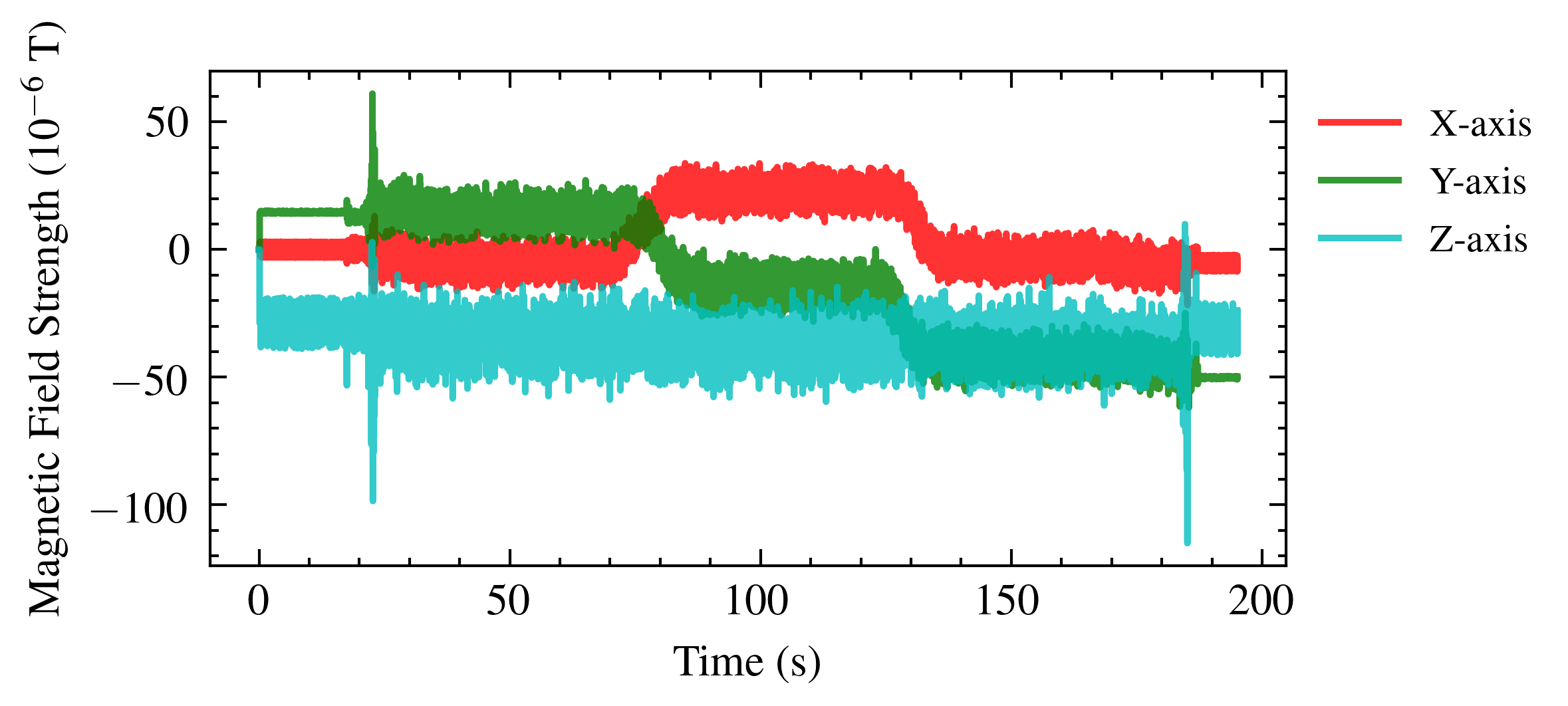}
  \caption{MARG sensor raw data collected in the experimental environment.}
  \label{fig:sub_b}
\end{figure}

\begin{figure}[htp]
    \centering
    \captionsetup{skip=0.3ex}
    \includegraphics[width=0.9\linewidth, keepaspectratio=false]{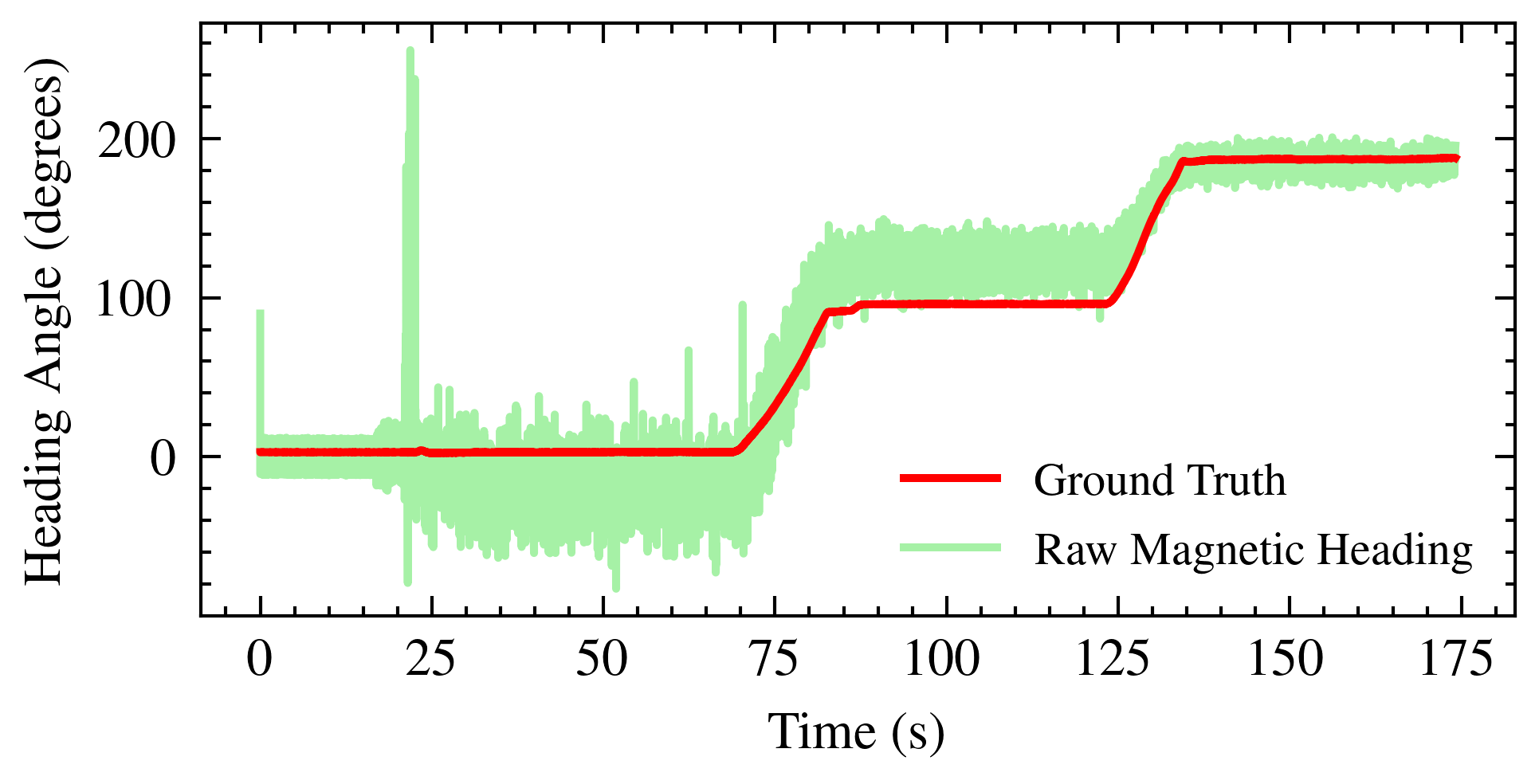}
    \caption{Comparison between UAV heading angles computed from disturbed raw magnetometer readings and ground truth headings provided by the motion capture system.}
    % \vspace{-0.3cm}
    \label{raw_mag_data_heading}
\end{figure}

Figure \ref{raw_mag_data_heading} compares the heading angle directly computed from raw magnetometer readings with the ground truth obtained from the motion capture system. During the first phase, the magnetometer-derived heading shows significant deviation from the ground truth, with errors reaching several tens of degrees at peak disturbance points. During the transition from first phase to second phase, the magnetometer output lags behind the actual UAV rotation, failing to accurately track the change in heading. In the second phase, a noticeable constant bias persists between the magnetometer-based heading and the ground truth. These results demonstrate that relying solely on magnetometer for heading estimation can lead to severe instability in UAV orientation, potentially posing risks during inspection tasks.

\subsection{Heading Estimation Accuracy Evaluation} \label{exp_c}

To validate the performance of the proposed AMO-HEAD approach, we conducted comparative tests of the progressively refined algorithm variants. 
The filter is initialized with state
$\mathbf x_0 = \left[1,\;0,\;0,\;0,\;0,\;0,\;0\right]^T$, 
representing a nominal unit quaternion and zero gyroscope bias, and with initial estimation error covariance
$\mathbf P_0 = 0.1\,\mathbf I_7$. 
The process noise covariance is initialized as $\mathbf Q_0 = \mathrm{diag}(0.02, 0.02, 0.02, 0.02, 10^{-6}, 10^{-6}, 10^{-6})$.
The initial acceleration measurement noise covariance is set to $\mathbf R_a = \mathrm{diag}(0.01, 0.01, 0.01)$ and magnetometer measurement noise covariance is set to $\mathbf R_m = \mathrm{diag}(0.1, 0.1, 0.1)$. 
The local reference magnetic field is set to $\mathbf m = \left[-2.81, 14.67,-28.64\right]^T$ in the inertial frame.
Four algorithm variants use the same initial values and
the estimation results compared with the ground truth heading are presented in Fig. \ref{yaw_proposed}. 
\begin{figure}[htp] 
    \centering
    \captionsetup{skip=0.3ex}
    \includegraphics[width=1.0\linewidth]{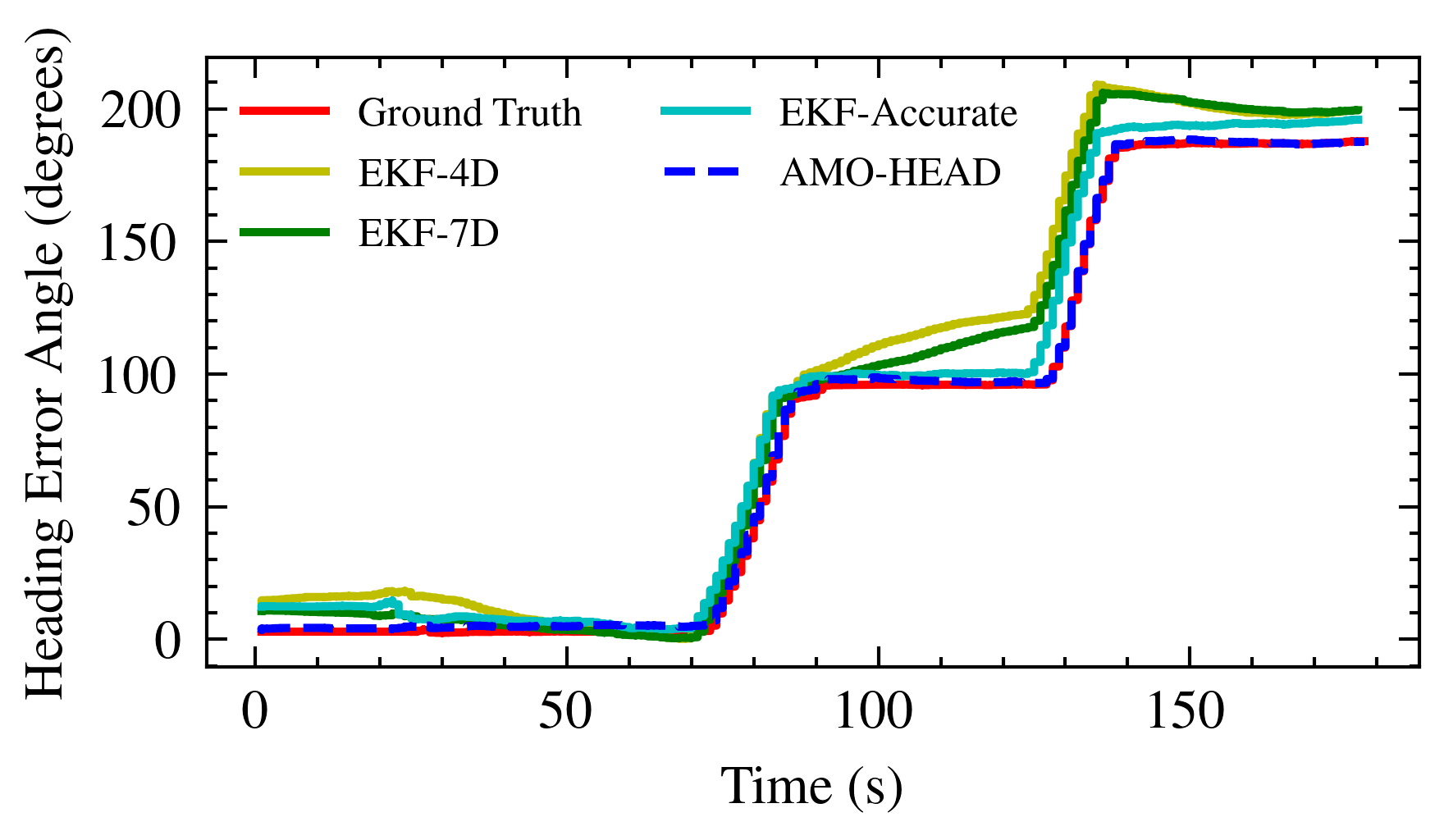}
    \caption{Comparison of four progressively refined heading estimation results with ground truth in the UAV self-rotation scenario.}
    % \vspace{-0.3cm}
    \label{yaw_proposed}
\end{figure}
From the initial stage, it can be seen that the proposed AMO-HEAD have faster convergence efficiency compared to three other variants.
In the first and third phase with severe magnetic disturbances, the EKF-4D algorithm demonstrates the largest deviations from the ground truth, exhibiting slow convergence.
For the EKF-7D variant, the estimation follows a similar overall trend to EKF-4D but shows improved alignment with the true heading. However, noticeable overshoot occurs at $T = 70s$, and substantial deviations remain around the second phase, accompanied by slow convergence. 
The EKF-Accurate algorithm, benefiting from detailed process noise modeling, achieves improved convergence at the second phase, yielding estimates generally close to the ground truth, although convergence at the first phase is slow. 
Finally, the proposed AMO-HEAD approach, indicated by the blue dashed curve, demonstrates superior accuracy with rapid convergence and close alignment to the ground truth across all phases. It clearly highlights its robust estimation in mitigating magnetic disturbances via adaptive covariance scaling.

To further quantify performance, we compare each heading estimate to the ground truth heading from the motion capture system and calculate the error. Specifically, we compute the heading difference through
\begin{equation}
    e_i(t) = \hat{\psi}_i(t) - \psi_{\text{GT}}(t),  
\end{equation}
where $\hat{\psi}_i(t)$ is the heading estimate from method $i \in$ $\{$EKF-4D, EKF-7D, EKF-Accurate, AMO-HEAD$\}$ at time $t$, and $\psi_{\text{GT}}(t)$ is the corresponding ground truth heading. 
\begin{figure}[htp]
    \centering
    \captionsetup{skip=0.3ex}
    \includegraphics[width=1.0\linewidth]{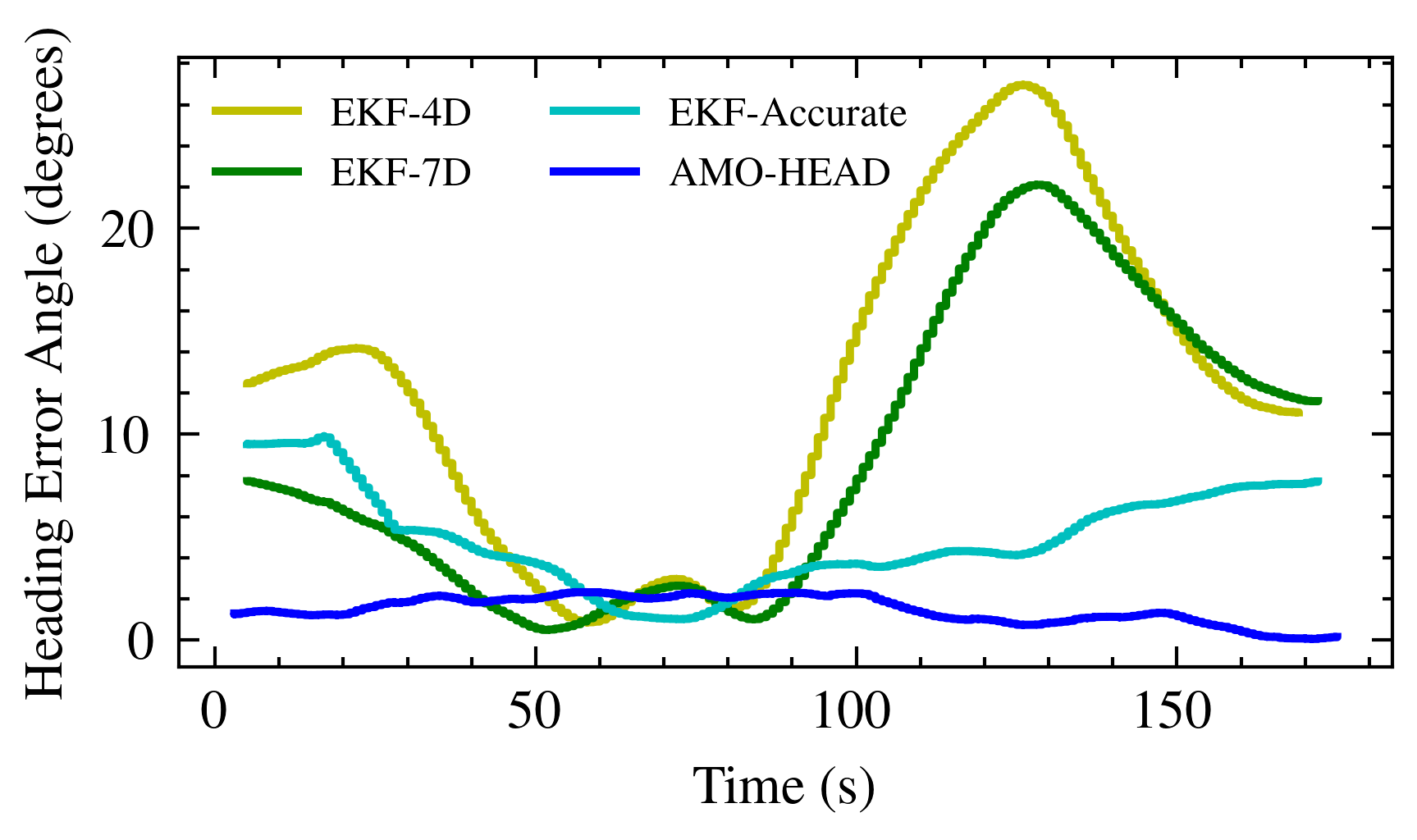}
    \caption{Estimation errors of four progressively refined algorithms in the UAV self-rotation scenario.}
    % \vspace{-0.3cm}
    \label{error_proposed}
\end{figure}
The errors are shown in Fig. \ref{error_proposed}. For the EKF-4D and EKF-7D methods, the heading estimation errors exhibit significant instability, with sudden jumps ranging from 4° to 30°. During approximately half of the experimental duration, the heading errors surpass 10°, presenting substantial risks to UAV stability and safe inspection tasks. Overall, EKF-Accurate demonstrates noticeably smaller errors compared to the previous two methods. 
However, during the periods of substantial magnetic disturbances at first and third phases, it still produces errors of around 8° to 10°, revealing its limited capability to resist severe magnetic interference. In contrast, the proposed AMO-HEAD consistently maintains low error levels around 2° throughout the experiment. Remarkably, its estimation error approaches zero for nearly half of the testing period, confirming exceptionally accurate and stable heading estimation performance.
% \begin{table}[!ht] 
%     \centering
%     \caption{RMSE of the Refined Estimation Algorithm}
%     \label{table_rmse}
%     \begin{tabular}{|c|c|c|c|C|}
%     \hline
%         \textbf{Method} & EKF-4D & EKF-7D & EKF-Accurate & AMO-HEAD \\ \hline
%         \textbf{RMSE} & 14.46° & 11.25° & 5.76° & 1.12° \\ \hline
%     \end{tabular}
% \end{table}
\begin{table} %开始表格环境
	\caption{Error Metrics of the Refined Estimation Algorithm}        %表格的标题
	\label{rmse_table}           %表格标签
	\begin{center}                 %center表示表格位置居中
		\begin{tabular}{c c c c c} %开始表格，c的数量表示列的数量，c表示居中
			\toprule               %三线表的第一条线
			\textbf{Method} &EKF-4D&EKF-7D&EKF-Accurate&AMO-HEAD\\       %表格第一行，&的位置对齐
			\midrule              
			\textbf{RMSE} & 14.55° & 11.91° & 6.26° & 1.58° \\
			\textbf{MAE} &12.15°&10.07°&6.01°&1.42°\\
			\bottomrule            %三线表的第三条线
		\end{tabular}
	\end{center}
\end{table}
To evaluate the overall accuracy of each heading estimation approach, we also calculate the root mean squared error (RMSE) and mean absolute error (MAE) of the heading outputs over the full experimental sequence and show the results in Table \ref{rmse_table}. The RMSE is defined as
\begin{equation}
\text{RMSE} = \sqrt{\frac{1}{N}\sum_{i=1}^{N}\bigl(\hat\psi_{i} - \psi_{GT}\bigr)^2}
\end{equation}
and measures the square-root of the average squared deviation between the estimated heading and the ground truth. The MAE is given by
\begin{equation}
    \text{MAE} = \frac{1}{N}\sum_{i=1}^{N}\bigl|\hat\psi_i - \psi_{GT}\bigr|
\end{equation}
and quantifies the average absolute deviation. Here, $N$ is the total number of samples in the experimental sequence. These two complementary metrics together provide an assessment of estimation error.
Consistent with the previously discussed heading estimation curves in Fig. \ref{yaw_proposed}, the EKF-7D method slightly outperforms EKF-4D, reducing RMSE by 2.64°, and lowering MAE from 12.15° to 10.07°. The EKF-Accurate method significantly surpasses EKF-7D, achieving a 47.4$\%$ reduction in RMSE and a corresponding MAE decrease to 6.01°. Finally, benefiting from refined process noise modeling and the adaptive measurement noise covariance scaling module, AMO-HEAD attains exceptionally low error values. It achieves an RMSE of only 1.58° and an MAE of 1.42°, notably falling below 2°.
% Consistent with the previously discussed heading estimation curves in Fig. \ref{yaw_proposed}, the EKF-7D method slightly outperforms EKF-4D, reducing RMSE by 2.64°. The EKF-Accurate method significantly surpasses EKF-7D, achieving a 47.4$\%$ reduction in RMSE. Finally, benefiting from refined process noise modeling and the adaptive measurement noise covariance scaling module, AMO-HEAD attains an exceptionally low RMSE value of only 2.24°.
\subsection{Comparison with Open-Source Estimation Algorithms} \label{comp_with_open_source}
To further evaluate the proposed AMO-HEAD approach, we conducted comparative experiments against three state-of-the-art (SOTA) heading estimation methods: VQF \cite{laidig2023vqf}, KOK \cite{kok2019fast}, and QMT \cite{seel2017eliminating}. 
{
This selection provides a comprehensive benchmark, as each algorithm represents a distinct and significant strength in the field. VQF serves as the primary high-accuracy benchmark. It has demonstrated superior performance in the extensive evaluations and has been validated with excellent results on UAV datasets. KOK, also recognized as a SOTA method in the VQF study, is included for its notable robustness against short-term magnetic field disturbances. Finally, QMT is chosen for its exceptional computational efficiency, offering one of the fastest implementations while maintaining strong accuracy.
}
\begin{figure}[htp]
    \centering
    \captionsetup{skip=0.3ex}
    \includegraphics[width=1.0\linewidth]{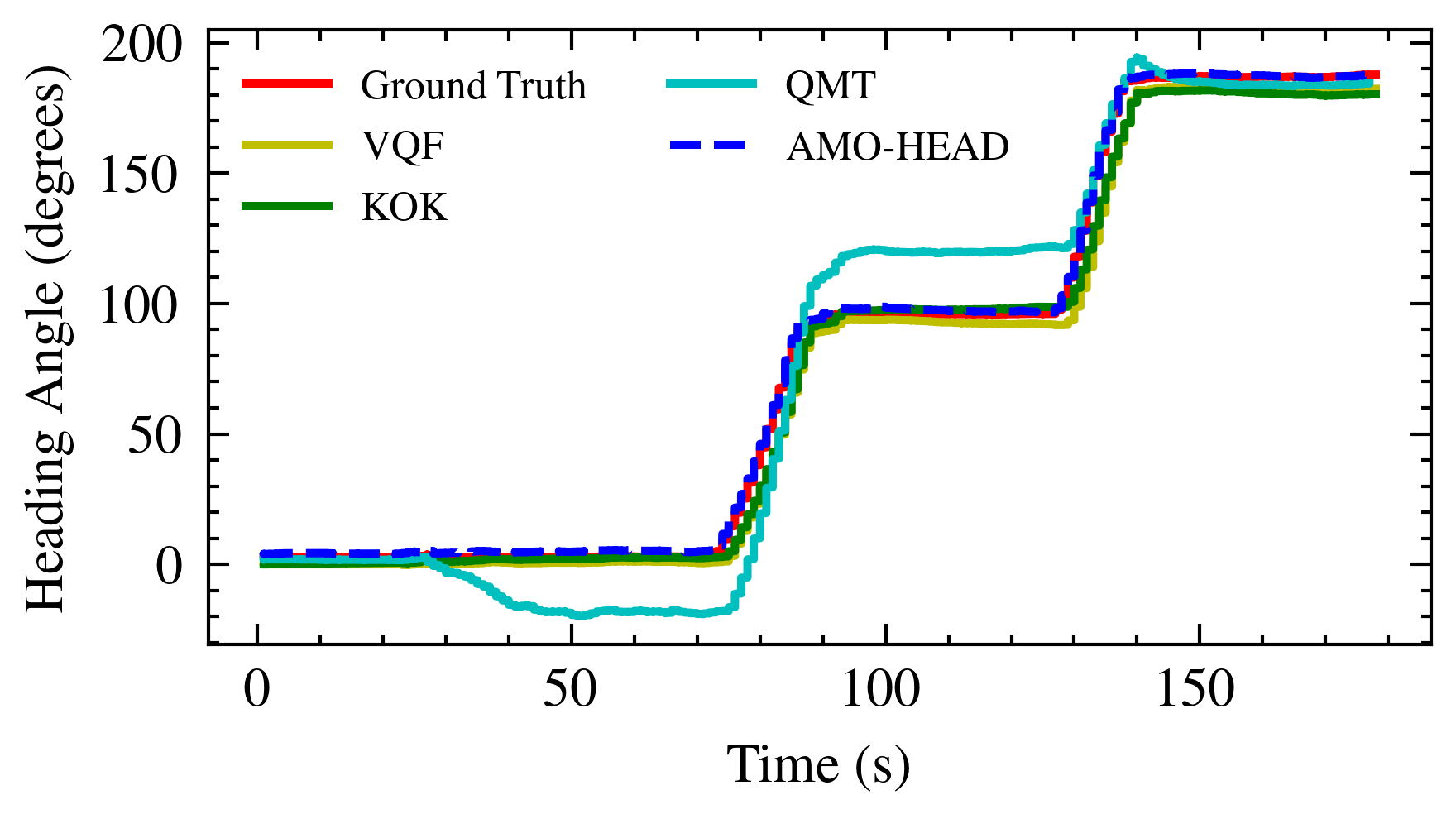}
    \caption{Comparison of heading estimates from the proposed AMO-HEAD approach and three open-source methods against the ground truth trajectory during UAV self-rotation.}
    % \vspace{-0.3cm}
    \label{open_source_compare}
\end{figure}
Specifically, in the VQF framework, strapdown gyroscope integration propagates an orientation quaternion to capture high frequency angular motion. A gravity-referenced tilt correction quaternion then removes roll and pitch drift using accelerometer data, and a scalar yaw adjustment accumulates magnetometer-derived heading corrections.
The initial attitude state is defined by setting both the gyroscope-integrated quaternion and the gravity-corrected quaternion to the identity quaternion $\left[1,\,0,\,0,\,0\right]^T$ and initializing the magnetometer-derived yaw correction angle to 0$^\circ$.
% VQF achieves robust 9D orientation estimation by first low-pass filtering accelerometer data in a near-inertial frame to isolate tilt, then modularly updating yaw with adaptive magnetometer fusion and bidirectional smoothing. 
The KOK method parameterizes orientation corrections as a minimal 3D rotation increment and applying a single fixed step gradient descent update to reduce computational load. The algorithm maintains two coupled state variables, an orientation quaternion $\mathbf q$ and a 3-D gyroscope bias $\mathbf b_g$, which are initialized to $\left[1,\,0,\,0,\,0\right]^T$ and $\left[0,\,0,\,0\right]^T$, respectively.
For QMT, it is a quaternion-based inertial motion tracking toolbox and we run the \texttt{oriEstIMU} algorithm in the toolbox. At each sampling interval, the bias-corrected angular rate is strapdown integrated to produce a predicted quaternion. Then it is refined by an accelerometer-based tilt correction and a magnetometer-based yaw correction. The initial state vector of QMT is same as KOK.

Importantly, these three methods rely solely on MARG sensor measurements as inputs and are processed with identical dataset collected from the experiment, allowing direct comparison against the ground truth. 
Figure \ref{open_source_compare} illustrates the heading angle estimations produced by each algorithm. 
Notably, the QMT algorithm displays obvious deviations and instability in the middle phase of the dataset. 
VQF and KOK follow the true heading reasonably well but still show noticeable errors, while AMO-HEAD stays closest to the ground truth trajectory. This demonstrates the robustness and accuracy of the proposed algorithm under challenging magnetic conditions.
\begin{figure}[htp]
    \centering
    \captionsetup{skip=0.3ex}
    \includegraphics[width=1.0\linewidth]{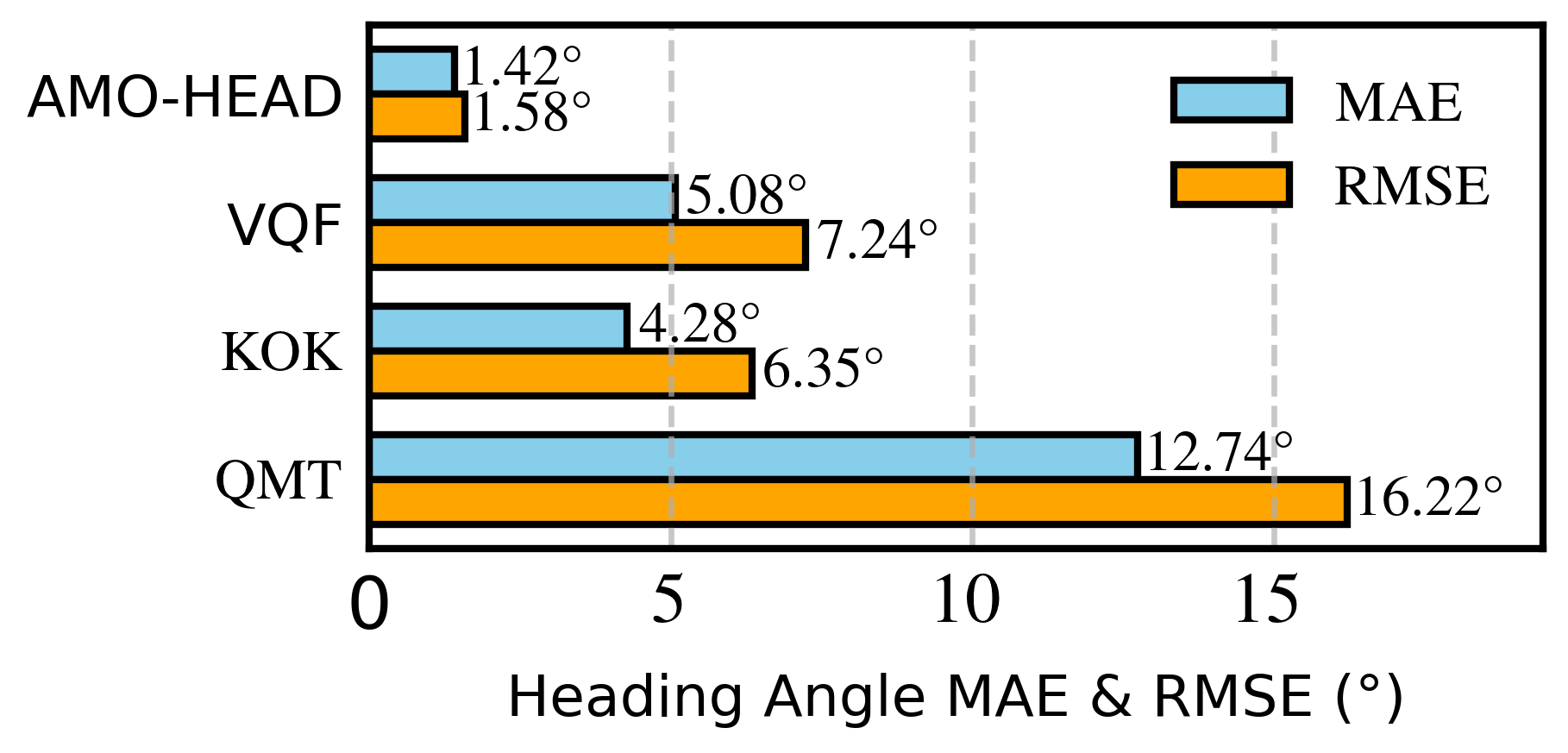}
    \caption{Heading estimation performance of the proposed AMO-HEAD and three open-source algorithms. Blue bars denote MAE and orange bars denote RMSE, with values in degrees.}
    % \vspace{-0.3cm}
    \label{open_source_metric}
\end{figure}
Quantitative results presented in Fig. \ref{open_source_metric} further reinforce these observations. 
AMO-HEAD achieves the lowest RMSE and MAE of all tested methods, with errors far below those of the three open-source benchmark algorithms. QMT shows the largest errors throughout the sequence. VQF updates its gyroscope bias only during detected stationary intervals. When a magnetometer reading is flagged as an outlier, the filter directly skips the magnetic correction step to reject that disturbance.
In comparison, AMO-HEAD's refined noise model is a fully continuous, data-driven adjustment that tracks gyro noise characteristics at all times. 
Meanwhile, AMO-HEAD avoids hard `all-or-nothing' cutoffs that lose potentially valuable information during moderate disturbances. AMO-HEAD automatically adapts its tuning to the varying disturbance levels.
The results validate that the proposed algorithm is a reliable and robust solution for real world UAV heading estimation tasks.

% Quantitative runtime analysis further complements the accuracy evaluation. To ensure consistency with the state-of-the-art comparison, we benchmarked AMO-HEAD against the same three open-source baselines (VQF, QMT, KOK) using unified ROS interfaces and identical datasets. Each algorithm's core filter step (prediction and update cycle) is instrumented to measure execution time on a laptop equipped with a 2.20 GHz Intel Core i7 CPU and 16 GB RAM. Two runtime metrics were recorded: the minimum single-step execution time and the average per-step runtime across the entire dataset. As summarized in Table~\ref{tab:sota_runtime_compare}, AMO-HEAD achieves the fastest single-step runtime of 0.162 ms, outperforming VQF and QMT. In terms of average runtime, AMO-HEAD performs comparably to the fastest baseline QMT, with a difference of less than 0.02 ms. Notably, the KOK baseline, whose open-source implementation is MATLAB-based, shows significantly higher runtime. These results highlight that AMO-HEAD not only achieves superior accuracy but also maintains computational efficiency suitable for real-time UAV deployment.
{
Quantitative runtime evaluation further complements the accuracy evaluation. To substantiate the lightweight design, we perform a runtime comparison on a laptop with 2.20 GHz Intel Core i7 CPU and 16 GB RAM. All four methods are wrapped under a unified ROS interface, and we instrument the core filter step to measure execution time. We have two metrics for the evaluation, the minimum single-step runtime and the average per-step runtime over the entire dataset. The results are summarized in Table \ref{tab:sota_runtime_compare}. AMO-HEAD achieves the fastest single-step runtime of 0.162 ms, outperforming VQF and QMT. In terms of average runtime, AMO-HEAD performs comparably to the fastest baseline QMT, with a difference of less than 0.02 ms. The KOK baseline, whose open-source implementation is MATLAB-based, exhibits obvious higher runtimes. These results highlight that AMO-HEAD not only achieves superior accuracy but also maintains computational efficiency suitable for real-time UAV deployment.}

\begin{table} %开始表格环境
	\caption{{Runtime Comparison of AMO-HEAD and Three Baseline Algorithms}}        %表格的标题
	\label{tab:sota_runtime_compare}           %表格标签
	\begin{center}                 %center表示表格位置居中
		\begin{tabular}{c c c c c} %开始表格，c的数量表示列的数量，c表示居中
			\toprule               %三线表的第一条线
			\textbf{Method} &AMO-HEAD&VQF&QMT&KOK\\       %表格第一行，&的位置对齐
			\midrule              
			\textbf{Single-step (ms)} & 0.162 & 0.223 & 0.190 & 1.943 \\
			\textbf{Per-step (ms)} &0.260&0.354&0.242&2.572\\
			\bottomrule            %三线表的第三条线
		\end{tabular}
	\end{center}
\end{table}

\subsection{{Long-Term Performance and Generalizability Validation}} 

\begin{table} %开始表格环境
	\caption{Error Metrics of the Long-Term Experiment}        %表格的标题
	\label{rmse_table_supply_exp}           %表格标签
	\begin{center}                 %center表示表格位置居中
		\begin{tabular}{c c c c c} %开始表格，c的数量表示列的数量，c表示居中
			\toprule               %三线表的第一条线
			\textbf{Method} &AMO-HEAD&VQF&QMT&KOK\\       %表格第一行，&的位置对齐
			\midrule              
			\textbf{RMSE} & 6.16° & 11.72° & 8.61° & 11.22° \\
			\textbf{MAE} &4.62°&9.74°&7.15°&8.35°\\
			\bottomrule            %三线表的第三条线
		\end{tabular}
	\end{center}
\end{table}

\begin{figure}[htp]
    \centering
    \captionsetup{skip=0.3ex}
    \includegraphics[width=1.0\linewidth]{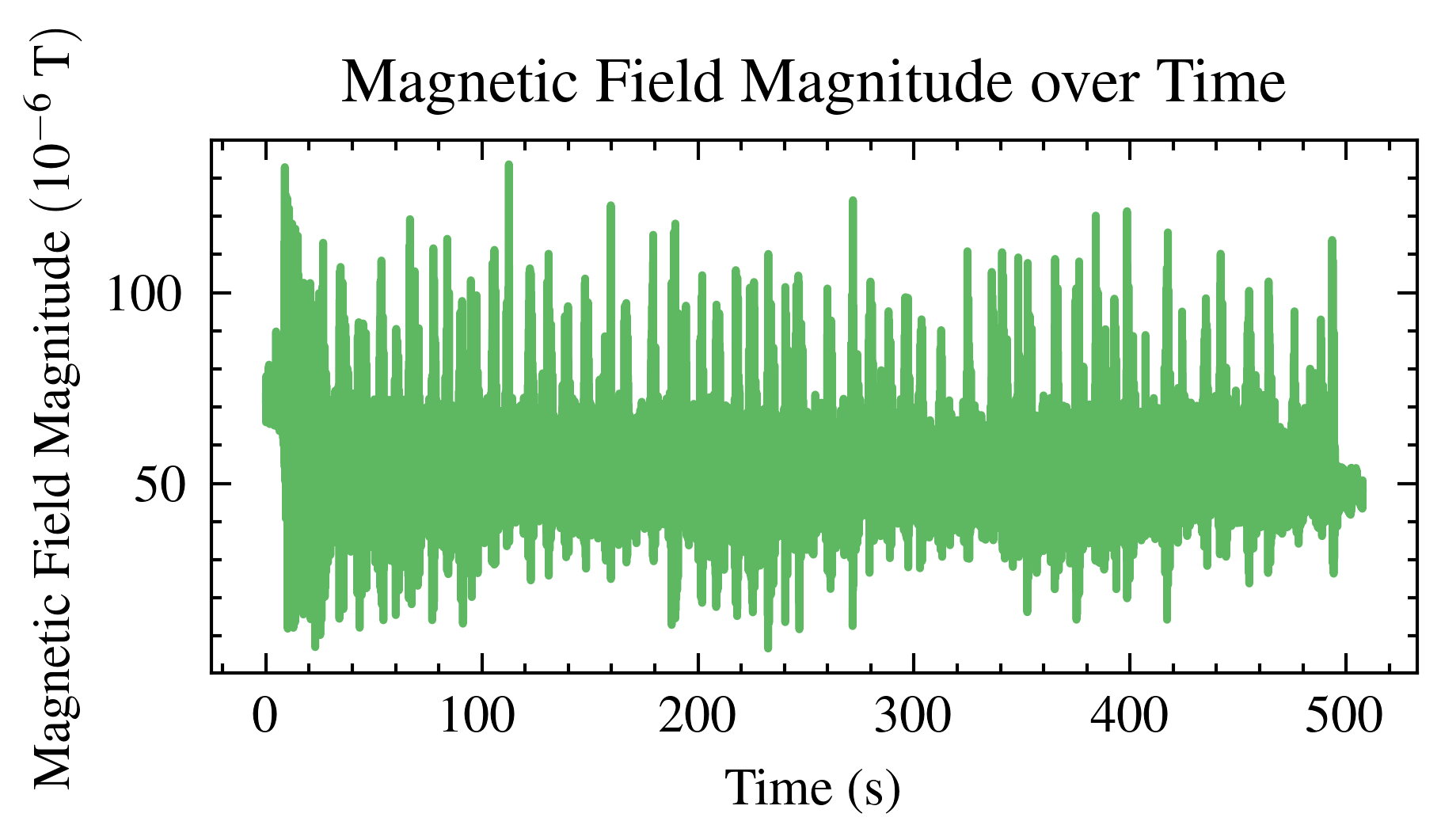}
    \caption{Magnetic field magnitude recorded during the 8-minute flight. The high volatility and sharp peaks indicate a magnetically challenging environment with strong, structured interference.}
    % \vspace{-0.3cm}
    \label{supply_exp_mag_field_strength}
\end{figure}

\begin{figure}[htp]
    \centering
    \captionsetup{skip=0.3ex}
    \includegraphics[width=1.0\linewidth]{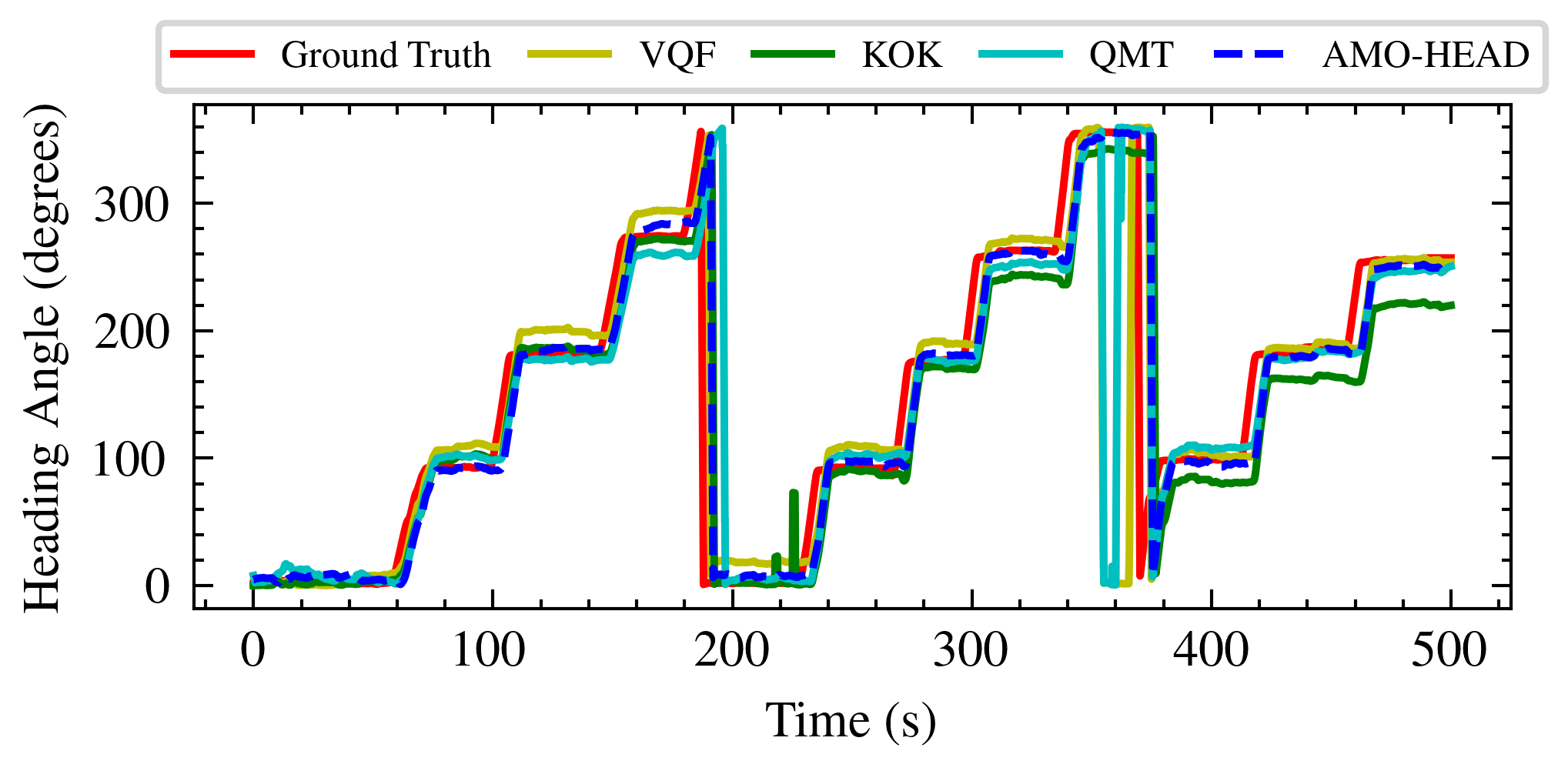}
    \caption{Comparison of heading estimates over the 8-minute long-term trajectory. AMO-HEAD demonstrates superior stability and tracking accuracy, closely following the ground truth without accumulating drift, unlike the benchmark algorithms.}
    % \vspace{-0.3cm}
    \label{supply_exp_head_compare}
\end{figure}

{To validate the long-term stability and generalizability of our approach, we conduct an additional experiment designed to test AMO-HEAD's robustness against accumulated drift and its performance in a distinct operational environment.
The experiment is consisted of a continuous, manually-piloted flight of approximately 8 minutes, with the duration limited by the UAV's battery endurance. We perform this test in an indoor area populated with electronic hardware and metallic infrastructure. As illustrated in Fig.~\ref{supply_exp_mag_field_strength}, this environment is characterized by strong, structured magnetic interference, with transient peak disturbances. The UAV follows a repetitive square trajectory, involving segments of straight-line motion followed by 90-degree heading turns. A motion capture system provides high-fidelity ground truth for the entire flight.}

{We benchmark the performance of AMO-HEAD against the same open-source methods. The heading estimation results are illustrated in Fig.~\ref{supply_exp_head_compare}. It is visually evident that AMO-HEAD remains consistently aligned with the ground truth, accurately tracking the repetitive turns without accumulating noticeable drift. In contrast, the other methods exhibit larger tracking errors and intermittent drift, with particularly significant deviations occurring during and after the sharp heading maneuvers.}

{Quantitative analysis, summarized in Table~\ref{rmse_table_supply_exp}, further substantiates these findings. AMO-HEAD achieves an RMSE of $6.16^\circ$ and an MAE of $4.62^\circ$, significantly outperforming all other methods. This result demonstrates two key strengths of our approach. First, the robust gyroscope bias estimation effectively prevents long-term drift, ensuring stable performance over long-duration missions. Second, the successful performance in a new, challenging environment confirms the generalizability of our adaptive framework, which automatically adjusts its sensor fusion strategy to match the ambient magnetic conditions.}

\section{Conclusion} \label{sec:conclusion}
In this paper, we formulate a MARG-only quaternion-based AEKF framework that advances UAV heading estimation, with incorporating gyroscope biases into the state vector. We introduce a refined process noise model that considers gyroscope measurement errors, slow bias drifts, and discretization errors induced by the integration using Euler method. Moreover, an adaptive measurement noise covariance matrix is implemented to suppress magnetic disturbances indoors. Furthermore, observability of the proposed model is proven using Lie derivative analysis. Additionally, our framework offers a computationally efficient implementation, achieving real-time performance on ARM Cortex-A7 processors. Experimental validation under severe indoor magnetic disturbances demonstrates a heading estimation RMSE of 1.58°, outperforming all the evaluated benchmark algorithms. Together, this comprehensive framework establishes a reliable model for resource-constrained aerial autonomy in indoor inspection tasks.
{
Furthermore, the core adaptive mechanism of our framework demonstrates significant potential for generalization. Its principle of dynamically weighting measurements offers a more robust solution than rigid rejection logic. It also provides a cost-effective alternative to expensive hardware, addressing common limitations found on other platforms such as ground and underwater vehicles. This makes our approach a promising candidate for robust heading estimation across the field of robotics.
}

\appendices
\section{{Background Definitions and Matrix Formulations}} \label{appendix:verbosity}
{
This appendix provides supplementary definitions and matrix formulations referenced in the main text. Specifically, we first define the coordinate systems used throughout the paper. We then present the full expressions for the key matrices in our AEKF-based framework: the quaternion-to-rotation matrix $\mathbf{C}(\mathbf{q})$, the state transition sub-matrix $\mathbf{A}_q$, and the measurement Jacobian matrix $\mathbf{H}_t$.}
\subsection{Coordinate Frame Definitions}
This work employs two primary coordinate systems: an inertial frame and a body frame. The inertial frame $\{n\}$ follows the east-north-up (ENU) convention, where its axes point towards east ($\mathbf{e}_E$), north ($\mathbf{e}_N$), and up ($\mathbf{e}_U$), respectively.
The body frame $\{b\}$ is rigidly attached to the UAV's airframe. The MARG sensor module is mounted such that its axes align with the UAV's forward, left, and upward directions. Consequently, the sensor frame axes and body frame axes ($\mathbf{e}_x$, $\mathbf{e}_y$, $\mathbf{e}_z$) are coincident, allowing any vector expressed in the sensor frame to be directly used in the body frame without additional transformation.
% Two coordinate systems are used in this paper, inertial and body coordinate frame systems. The inertial frame $\{n\}$ is defined as the east-north-up (ENU) convention, where $x$-, $y$-, $z$-axes point east $\mathbf e_E$, north $\mathbf e_N$, and up $\mathbf e_U$ respectively. 
% The body frame $\{b\}$ is rigidly attached to the UAV airframe, while the sensor frame is defined by the MARG sensor module itself. Because the MARG unit is mounted on the flight controller with its axes aligned to the UAV's forward, left, and upward directions, the sensor frame axes and body frame axes $\mathbf e_x$, $\mathbf e_y$, $\mathbf e_z$ coincide. Consequently, any vector expressed in the sensor frame is directly valid in the body frame without additional transformation.

\subsection{Filter Matrix Formulations}
First, the rotation matrix $\mathbf{C}(\mathbf{q})$, which transforms a vector from the inertial frame $\{n\}$ to the body frame $\{b\}$, is formulated from the unit quaternion $\mathbf{q}$ as follows
\begin{equation} \label{eq:rotation_matrix}
\begin{aligned}
& \mathbf {C(q)} = \\ 
& \resizebox{0.95\linewidth}{!}{
    $\begin{bmatrix}
        q^2_w+q^2_x-q^2_y-q^2_z & 2(q_xq_y+q_wq_z) & 2(q_xq_z-q_wq_y) \\
        2(q_xq_y-q_wq_z) & q^2_w-q^2_x+q^2_y-q^2_z & 2(q_yq_z+q_wq_x) \\
        2(q_xq_z+q_wq_y) & 2(q_yq_z-q_wq_x) & q^2_w-q^2_x-q^2_y+q^2_z
    \end{bmatrix}$
}
\end{aligned}
\end{equation}
% \subsubsection{State Transition and Measurement Jacobian Matrices}

Next, the state transition sub-matrix for the quaternion kinematics, $\mathbf{A}_q(\boldsymbol{\tilde{\omega}}_t, \Delta t)$, derived via a first-order Euler approximation, is defined as
% The matrix $\mathbf{A}_q(\boldsymbol{\tilde{\omega}}_t, \Delta t)$ is determined by
\begin{equation}
\begin{aligned}
    \mathbf{A}_q(\boldsymbol{\tilde{\omega}}_t, \Delta t) &=\mathbf{I}_4 + \frac{\Delta t}{2} \mathbf{\Omega}(\boldsymbol{\tilde{\omega}}_t)\\
    & =\begin{bmatrix}
1 & -\frac{\Delta t}{2} \tilde{\omega}_x & -\frac{\Delta t}{2} \tilde{\omega}_y & -\frac{\Delta t}{2} \tilde{\omega}_z \\
\frac{\Delta t}{2} \tilde{\omega}_x & 1 & \frac{\Delta t}{2} \tilde{\omega}_z & -\frac{\Delta t}{2} \tilde{\omega}_y \\
\frac{\Delta t}{2} \tilde{\omega}_y & -\frac{\Delta t}{2} \tilde{\omega}_z & 1 & \frac{\Delta t}{2} \tilde{\omega}_x \\
\frac{\Delta t}{2} \tilde{\omega}_z & \frac{\Delta t}{2} \tilde{\omega}_y & -\frac{\Delta t}{2} \tilde{\omega}_x & 1
\end{bmatrix}
\end{aligned}
\end{equation}

Finally, the full expression for the measurement Jacobian matrix $\mathbf{H}_t$, evaluated at the prior state estimate $\mathbf{x}_t^-$, is given in Eq.~\eqref{eq:H_matrix} at the top of the next page. The first three rows of $\mathbf{H}_t$ are the partial derivatives with respect to the accelerometer model, while the last three rows are with respect to the magnetometer model.

\begin{figure*}[t!]
  \centering
\begin{equation} 
    \label{eq:H_matrix}
    \mathbf{H}_t = 
    \begin{bmatrix}
    -2gq_y                    & 2gq_z                     & -2gq_w                    & 2gq_x                     & 0 & 0 & 0\\
    2gq_x                     & 2gq_w                     & 2gq_z                     & 2gq_y                     & 0 & 0 & 0\\
    2gq_w                     & -2gq_x                    & -2gq_y                    & 2gq_z                     & 0 & 0 & 0\\[6pt]
    2m^r_x q_w + 2m^r_z q_y   & 2m^r_x q_x + 2m^r_z q_z   & -2m^r_x q_y + 2m^r_z q_w  & -2m^r_x q_z + 2m^r_z q_x  & 0 & 0 & 0\\
    2m^r_x q_z - 2m^r_z q_x   & 2m^r_x q_y - 2m^r_z q_w   & 2m^r_x q_x + 2m^r_z q_z  & 2m^r_x q_w + 2m^r_z q_y   & 0 & 0 & 0\\
    -2m^r_x q_y + 2m^r_z q_w  & 2m^r_x q_z - 2m^r_z q_x  & -2m^r_x q_w - 2m^r_z q_y & 2m^r_x q_x + 2m^r_z q_z  & 0 & 0 & 0
    \end{bmatrix}
    \end{equation}
\end{figure*}

\section{Proof of Proposition 1} \label{proof_ob_matrix}
{
To enhance clarity, we first outline the high-level overview.
We begin by assuming that the system operates under nonsingular attitude conditions. Singular attitudes in this study refer to configurations where the UAV roll angle $\phi$ or pitch angle $\theta$ equals $\pm 90^\circ$.}
{
We first select the $7 \times 7$ submatrix $\mathbf{M}$ from the complete observability matrix.
Under the nonsingular-attitude assumption, we examine two alternative row combinations of $\mathbf A$ and show that its determinant is nonzero across all nonsingular attitude configurations.
We then prove, by contradiction, that the determinant of $\mathbf{B}$ never vanishes.
Finally, leveraging the determinant property of block matrices, we establish that the entire submatrix $\mathbf{M}$ is full rank, thereby confirming the system's local observability.
}The proof of Proposition 1 will be presented by the following four sequential steps. 

\subsubsection{Submatrix Selection and Decomposition}
Based on the rank minor theorem \cite{Horn_Johnson_2012}, the rank of a matrix is determined by the size of its largest nonzero minor, i.e., the $12 \times 7$ dimension of $\mathbf{\Xi}$ implies that its maximum possible nonzero minor is of size $7 \times 7$. A selected $7 \times 7$ minor which has a nonzero determinant would establish linear independence among its rows. To demonstrate the rank property, we select a $7 \times 7$ submatrix $\mathbf M$ of $\mathbf \Xi$ by choosing rows $\{$1,2,3,4,7,8,9$\}$ and columns $\{$1,2,3,4,5,6,7$\}$, denoted as
\begin{equation}
\mathbf M=
\begin{bmatrix}
\mathbf  A & \mathbf  0 \\
\mathbf  C & \mathbf  B
\end{bmatrix}
\end{equation}
where
\begin{equation}
\begin{aligned}
\mathbf  A &= \begin{bmatrix}
-2gq_y & 2gq_z & -2gq_w & 2gq_x \\
2gq_x  & 2gq_w  & 2gq_z  & 2gq_y \\
2gq_w  & -2gq_x & -2gq_y & 2gq_z \\
\Xi_{4,1} & \Xi_{4,2} & \Xi_{4,3} & \Xi_{4,4}
\end{bmatrix} \\
\mathbf B &= \begin{bmatrix}
0           & \Xi_{7,6} & \Xi_{7,7} \\
\Xi_{8,5}   & 0         & \Xi_{8,7} \\
\Xi_{9,5}   & \Xi_{9,6} & 0
\end{bmatrix}\\
\mathbf C &= \begin{bmatrix}
\Xi_{7,1}    & \Xi_{7,2} & \Xi_{7,3} & \Xi_{7,4} \\
\Xi_{8,1}    & \Xi_{8,2} & \Xi_{8,3} & \Xi_{8,4}  \\
\Xi_{9,1}    & \Xi_{9,2} & \Xi_{9,3} & \Xi_{9,4} 
\end{bmatrix}
\end{aligned}
\end{equation}
where $\Xi_{i,j}$ denotes the $i$-th row and $j$-th column elements of the observability matrix $\mathbf \Xi$ and the required elements in the selected submatrix can be readily obtained by some tedious but straightforward mathematical derivations as
\begin{equation*}
    \left\{
    \begin{array}{ll}
        \Xi_{4,1} =  2q_w - q_y \\
        \Xi_{4,2} =  2q_x + q_z \\
        \Xi_{4,3} = -2q_y - q_w \\
        \Xi_{4,4} = -2q_z + q_x
    \end{array}
    \right.
    \quad
        \left\{
    \begin{array}{ll}
        \Xi_{7,1} = -2g(\tilde \omega_y q_w - \tilde \omega_z q_x) \\
        \Xi_{7,2} = 2g(\tilde \omega_z q_w + \tilde \omega_y q_x) \\
        \Xi_{7,3} = 2g(\tilde \omega_y q_y + \tilde \omega_z q_z) \\
        \Xi_{7,4} = 2g(\tilde \omega_z q_y - \tilde \omega_y q_z) \\
        \Xi_{7,6} = -g(q_w^2 + q_z^2 - q_x^2 - q_y^2) \\
        \Xi_{7,7} = -2g(q_w q_x + q_y q_z)
    \end{array}
    \right.
    \end{equation*}
\begin{equation*}
\left\{
\begin{array}{ll}
\Xi_{8,1} = 2g(\tilde \omega_x q_w + \tilde \omega_z q_y)\\
\Xi_{8,2} = -2g(\tilde \omega_x q_x + \tilde \omega_z q_z)\\
\Xi_{8,3} = 2g(\tilde \omega_z q_w - \tilde \omega_x q_y)\\
\Xi_{8,4} = 2g(\tilde \omega_x q_z - \tilde \omega_z q_x)\\
\Xi_{8,5} = -g(q_x^2 - q_w^2 - q_z^2 + q_y^2)\\
\Xi_{8,7} = -2g(q_x q_z - q_w q_y)
\end{array}
\right.
%\end{equation*}
%\begin{equation*}
\left\{
\begin{array}{ll}
\Xi_{9,1} = -2g(\tilde \omega_x q_x + \tilde \omega_y q_y)\\
\Xi_{9,2} = -2g(\tilde \omega_x q_w - \tilde \omega_y q_z)\\
\Xi_{9,3} = -2g(\tilde \omega_y q_w + \tilde \omega_x q_z)\\
\Xi_{9,4} = 2g(\tilde \omega_y q_x - \tilde \omega_x q_y)\\
\Xi_{9,5} = -2g(q_w q_x + q_y q_z)\\
\Xi_{9,6} -2g(q_w q_y - q_x q_z)\\
\end{array}
\right.
\end{equation*}

\subsubsection{Proof that the determinant of matrix $\mathbf A$ is nonzero}
After performing the derivation by cofactor expansion and assembling the resulting minors, we obtain
\begin{equation} \label{deta_expression}
\det (\mathbf A) = -32g^3(q_w q_z + q_x q_y)
\end{equation}

We express Eq. \eqref{deta_expression} using trigonometric functions. Applying the standard Z-Y-X Euler angle to quaternion conversion \cite{kuipers1999quaternions}, we have
\begin{equation}
\begin{aligned}
q_w &= c_\phi\,c_\theta\,c_\psi + s_\phi\,s_\theta\,s_\psi\\
q_x &= s_\phi\,c_\theta\,c_\psi - c_\phi\,s_\theta\,s_\psi\\
q_y &= c_\phi\,s_\theta\,c_\psi + s_\phi\,c_\theta\,s_\psi\\
q_z &= c_\phi\,c_\theta\,s_\psi - s_\phi\,s_\theta\,c_\psi
\end{aligned}
\end{equation}
where 
$c_\alpha = \cos\frac{\alpha}{2},\quad
s_\alpha = \sin\frac{\alpha}{2},
\quad \alpha\in\{\psi,\theta,\phi\} $
and $\psi$, $\theta$, $\phi$ denote heading, pitch, and roll respectively.
Then, it can be readily derived that
\begin{equation}
\begin{aligned}
q_w\,q_z 
&= \bigl(c_\phi c_\theta c_\psi + s_\phi s_\theta s_\psi\bigr)
   \bigl(c_\phi c_\theta s_\psi - s_\phi s_\theta c_\psi\bigr)\\
&= c_\phi^2c_\theta^2\,c_\psi s_\psi
 - c_\phi s_\phi c_\theta s_\theta\,c_\psi^2
 + s_\phi c_\phi s_\theta c_\theta\,s_\psi^2
 - s_\phi^2s_\theta^2\,s_\psi c_\psi
\\
q_x\,q_y
&= \bigl(s_\phi c_\theta c_\psi - c_\phi s_\theta s_\psi\bigr)
   \bigl(c_\phi s_\theta c_\psi + s_\phi c_\theta s_\psi\bigr)\\
&= s_\phi c_\phi c_\theta s_\theta\,c_\psi^2
 + s_\phi^2c_\theta^2\,c_\psi s_\psi
 - c_\phi^2s_\theta^2\,s_\psi c_\psi
 - s_\phi c_\phi c_\theta s_\theta\,s_\psi^2
\end{aligned}
\end{equation}

After some mathematic manipulations, we can obtain
\begin{equation} \label{q_conversion}
q_wq_z + q_xq_y 
= \cos\theta\;c_\psi s_\psi
= \frac{1}{2} \cos\theta\,\sin\psi
\end{equation}

{Substituting Eq.~\eqref{q_conversion} into Eq.~\eqref{deta_expression} gives the determinant of matrix $\mathbf A$ as
}
\begin{equation}
\det (\mathbf A) = -16\,g^3\cos\theta\sin\psi
\end{equation}

Considering the quaternion constraint $q_w^2 + q_x^2 + q_y^2 + q_z^2=1$, it follows that $\det(\mathbf A)=0$ holds only if $\cos\theta\,\sin\psi=0$.  Equivalently, $\det (\mathbf A)$ vanishes only if the yaw angle is $\psi=0^ \circ$ or $\psi=180^ \circ$, or the pitch angle is $\theta=\pm 90^ \circ$. In practical UAV inspection operations, the UAV's pitch angle never reaches the extreme value of $90^ \circ$. Next, we prove that $\psi=0^ \circ$ or $\psi=180^ \circ$ does not influence the observability of the system. For this purpose, we choose rows $\{$1,2,3,5,7,8,9$\}$ of $\mathbf \Xi$. Notice that matrices $\mathbf B$ and $\mathbf C$ remain the same for this second choice of matrix $\mathbf A$, which is given by 
\begin{equation}
\mathbf  A = \begin{bmatrix}
-2gq_y & 2gq_z & -2gq_w & 2gq_x \\
2gq_x  & 2gq_w  & 2gq_z  & 2gq_y \\
2gq_w  & -2gq_x & -2gq_y & 2gq_z \\
-2q_z + q_x & 2q_y + q_w & 2q_x + q_z & -2q_w + q_y
\end{bmatrix}
\end{equation}
which directly gives
\begin{equation}
\det (\mathbf A) = 16g^3(1-2q_w^2-2q_y^2)
\end{equation}

After some mathematical derivations, we have
\begin{equation}
\det (\mathbf A) = -16 \, g^3 (\cos \phi \cos \psi + \sin \theta \sin \phi \sin \psi)
\end{equation}

Define $C=\cos \phi,\,D=\sin\theta\,\sin\phi$,
we rewrite
\begin{equation}
\begin{aligned}
\det (\mathbf A) &=-16g^3\,(C\cos\psi+D\sin\psi) \\&=-16g^3\sqrt{C^2+D^2}\,\cos\bigl(\psi-\delta\bigr)
\end{aligned}
\end{equation}
{
where $\delta=\arctan(D,C)$. Here, $\arctan(\cdot, \cdot)$ denotes the two-argument arctangent function  \cite{zhang2017design}, which ensures correct angle determination for all values of $C$ and $D$, including the cases $C \le 0$.
}
Considering the gravity acceleration $g > 0$, it follows that $\det(\mathbf A)=0$ holds when $\sqrt{C^2+D^2}\,\cos\bigl(\psi-\delta\bigr) = 0$. 
Note that
$\sqrt{C^2+D^2}=\sqrt{\cos^2\phi+\sin^2\theta\,\sin^2\phi}>0$
for all $(\phi, \theta)$ except the singular set $\{\cos\phi=0,\;\sin\theta=0\}$. 
In practical UAV inspection operations, the UAV's roll angle never reaches the extreme value of $90^ \circ$ thus $\cos\phi \neq 0$. 
Therefore, $\det (\mathbf A)=0$ if and only if $\cos(\psi-\delta)=0$. If $\psi=0^ \circ$ or $\psi=180^ \circ$, $\det (\mathbf A)=0$ implies $\cos(-\delta)=0$ or $\cos(\delta)=0$. However, since $\frac{D}{C}=\sin\theta \tan\phi$ is bounded due to the physical limitations of the roll angle, i.e., it cannot reaches to $\pm 90^{\circ}$, it can be concluded that $\cos(-\delta)\ne0$ or $\cos(\delta)\ne0$. This proves that $\det (\mathbf A)\ne0$ when $\psi=0^ \circ$ or $\psi=180^ \circ$ for the second choice of matrix $\mathbf A$.

\subsubsection{Proof that the determinant of matrix $\mathbf B$ is nonzero}
Similarly, evaluating the determinant of matrix $\mathbf B$ gives
\begin{equation}\label{bd}
\begin{aligned}
\det (\mathbf B) 
&= \Xi_{7,6} \, \Xi_{8,7} \, \Xi_{9,5} + \Xi_{7,7} \, \Xi_{8,5} \, \Xi_{9,6} \\
&=8g^3(q_x^2+q_y^2-q_w^2-q_z^2)(q_w q_x+q_y q_z)(q_x q_z - q_w q_y)
\end{aligned}
\end{equation}

Suppose that $\det (\mathbf B) = 0$, then it can be concluded that all three terms in the brackets of Eq. \eqref{bd} should simultaneously become zero. We thus consider this condition along with the quaternion normalization constraint as
\begin{subequations}\label{det_b_lianli}
\begin{align}
q_x^2 + q_y^2 - q_w^2 - q_z^2 &= 0 \label{det_b_lianli_a}\\
q_w q_x + q_y q_z &= 0 \label{det_b_lianli_b}\\
q_x q_z - q_w q_y &= 0 \label{det_b_lianli_c}\\
q_w^2 + q_x^2 + q_y^2 + q_z^2 &= 1 \label{det_b_lianli_d}
\end{align}
\end{subequations}

From the Eq. \eqref{det_b_lianli_a} and Eq. \eqref{det_b_lianli_d}, we have 
\begin{equation}
q_x^2 + q_y^2 = q_w^2 + q_z^2
\end{equation}

Substituting the preceding equation into the normalization constraint \eqref{det_b_lianli_d} yields $q_w^2 + q_z^2 = \frac{1}{2}$. Consequently, $q_x^2 + q_y^2 = \frac{1}{2}$. Define two vectors in $\mathbb{R}^2$ as $\mathbf{a}=[q_w,q_z]$ and $\mathbf{b}=[q_x,q_y]$. Then, Eq. \eqref{det_b_lianli_b} can be expressed as $\mathbf{a}\cdot\mathbf{b}^T=0$. Similarly, Eq. \eqref{det_b_lianli_c} becomes $\mathbf{b}\cdot\mathbf{a}_\perp^T=0$, where $\mathbf{a}_\perp=[q_z,-q_w]$ is orthogonal to $\mathbf{a}$. Since 
\begin{equation}
|\mathbf{a}|=\sqrt{q_w^2+q_z^2}=\sqrt{\frac{1}{2}}>0
\end{equation}
vectors $\mathbf{a}$ and $\mathbf{a}_\perp$ form a basis for $\mathbb{R}^2$. However, vector $\mathbf{b}$ is orthogonal to both basis vectors, implying $\mathbf{b}=\mathbf{0}$. This directly violates the condition that 
\begin{equation}
|\mathbf{b}|^2=q_x^2+q_y^2=\frac{1}{2}>0
\end{equation}

Therefore, it can be concluded that $\det(\mathbf{B})\neq0$ for all physically valid quaternions by contradiction. 
% From Eq. \eqref{det_b_lianli_a}, we can get \( q_x^2 + q_y^2 = q_w^2 + q_z^2 \). Substituting it into Eq. \eqref{det_b_lianli_d}
% \begin{equation}
% (q_w^2 + q_z^2) + (q_w^2 + q_z^2) = 1 \implies q_w^2 + q_z^2 = \frac{1}{2}
% \end{equation}
% Resubstitute into Eq. \eqref{det_b_lianli_a}, we have 
% $q_x^2 + q_y^2 = \frac{1}{2}$.
% Both sums of squares are positive, so \((q_w, q_z) \neq (0, 0)\) and \((q_x, q_y) \neq (0, 0)\).
% Now define vectors that \(\mathbf{a} = (q_w, q_z)\) and \(\mathbf{b} = (q_x, q_y)\). Eq. \eqref{det_b_lianli_b} can be written as \(\mathbf{a} \cdot \mathbf{b} = 0\) and Eq. \eqref{det_b_lianli_c} is
% \begin{equation}
% q_x q_z - q_w q_y = \mathbf{b} \cdot (q_z, -q_w) = \mathbf{b} \cdot \mathbf{a}^{\perp} = 0
% \end{equation}
% where \(\mathbf{a}^{\perp} = (q_z, -q_w)\) is orthogonal to vector \(\mathbf{a}\). Since \(|\mathbf{a}| = \sqrt{q_w^2 + q_z^2} = \sqrt{\frac{1}{2}} > 0\), \(\mathbf{a}\) and \(\mathbf{a}^{\perp}\) form a basis for \(\mathbb{R}^2\). 
% Thus \(\mathbf{b}\) is orthogonal to both vectors implies \(\mathbf{b} = \mathbf{0}\), contradicting \(|\mathbf{b}|^2 = q_x^2 + q_y^2 = \frac{1}{2} \).
% The contradiction shows that the system of four equations has no real solution under the constraints above. Consequently, $\det (\mathbf B)$ remains strictly nonzero for all physically valid quaternions.

\subsubsection{Conclusion on Observability}
According to the determinant of block triangular matrix theorem \cite{Horn_Johnson_2012}, the determinant of the selected submatrix is 
\begin{equation}
\det(\mathbf M)=\det(\mathbf A)\det(\mathbf B) \neq 0
\end{equation}

As long as the UAV does not continuously remain at the specific orientations described above. Hence, the observability matrix $\mathbf{\Xi}$ satisfies $\operatorname{rank}(\mathbf{\Xi})=7$.
This proves Proposition \ref{prop:fullrank}, establishing the local observability of the system under typical operational conditions.
{
Despite the theoretical possibility of rank deficiency in the observability matrix at specific singular attitudes (e.g., when the heading angle is $0^\circ, 90^\circ, \text{or } 180^\circ$), these points do not compromise the filter's practical performance or lead to instability. This practical robustness stems from two primary factors. On the one hand, a UAV is an inherently dynamic system subject to constant, minor corrective movements. The UAV passes these singular attitudes transiently that the filter is not exposed to any single point long enough to cause divergence. On the other hand, our observability verification method is itself robust. As demonstrated in the analysis, global observability is confirmed by verifying the rank of different submatrices. A rank drop in one specific submatrix does not imply a total loss of observability, as another submatrix can still retain full rank. This conclusion is corroborated in the subsequent experimental section, where no filter divergence or noticeable performance degradation is observed when the UAV's trajectory passed through or near these theoretical singular attitudes.
}
% \end{proof}
\vspace*{3pt} 
% \section*{Acknowledgment}

% The preferred spelling of the word ``acknowledgment'' in America is without 
% an ``e'' after the ``g''. Avoid the stilted expression ``one of us (R. B. 
% G.) thanks $\ldots$''. Instead, try ``R. B. G. thanks$\ldots$''. Put sponsor 
% acknowledgments in the unnumbered footnote on the first page.

% \section*{References}

% Please number citations consecutively within brackets \cite{b1}. The 
% sentence punctuation follows the bracket \cite{b2}. Refer simply to the reference 
% number, as in \cite{b3}---do not use ``Ref. \cite{b3}'' or ``reference \cite{b3}'' except at 
% the beginning of a sentence: ``Reference \cite{b3} is the first $\ldots$''

% Number footnotes separately in superscripts. Place the actual footnote at 
% the bottom of the column in which it is cited. Do not put footnotes in the 
% abstract or reference list. Use letters for table footnotes.

% Unless there are six authors or more give all authors' names; do not use 
% ``et al.''. Papers that have not been published, even if they have been 
% submitted for publication, should be cited as ``unpublished'' \cite{b4}. Papers 
% that have been accepted for publication should be cited as ``in press'' \cite{b5}. 
% Capitalize only the first word in a paper title, except for proper nouns and 
% element symbols.

% For papers published in translation journals, please give the English 
% citation first, followed by the original foreign-language citation \cite{b6}.

\begingroup
\small
\bibliographystyle{IEEEtranN} 
\bibliography{customabrv,IEEEabrv,paper}
\endgroup

% \bibliographystyle{mybstfile}
% \bibliography{IEEEabrv,mybibfile}

% \bibitem{review1}M. Nazarahari and H. Rouhani, “40 years of sensor fusion for orientation tracking via magnetic and inertial measurement units: Methods, lessons learned, and future challenges,” Inf. Fusion, vol. 68, pp. 67-84, 2021, doi: 10.1016/j.inffus.2020.10.018.
% \bibitem{cf1}R. G. Valenti, I. Dryanovski, and J. Xiao, “Keeping a Good Attitude: A Quaternion-Based Orientation Filter for IMUs and MARGs,” Sensors, vol. 15, no. 8, pp. 19302-19330, 2015, doi: 10.3390/s150819302.

\vspace{12pt}
% \color{red}
% IEEE conference templates contain guidance text for composing and formatting conference papers. Please ensure that all template text is removed from your conference paper prior to submission to the conference. Failure to remove the template text from your paper may result in your paper not being published.

\end{document}